\documentclass[journal]{IEEEtran}
\ifCLASSINFOpdf
\else
\fi
\usepackage[utf8]{inputenc} 
\usepackage[T1]{fontenc}    
\usepackage{hyperref}       
\usepackage{url}            
\usepackage{booktabs}       
\usepackage{amsfonts}       
\usepackage{nicefrac}       
\usepackage{microtype}      

\usepackage{amsthm}
\usepackage{amssymb}
\usepackage{amsmath}
\usepackage{graphicx}
\usepackage{subfigure}
\usepackage{hyperref}
\usepackage{color}

\theoremstyle{definition}
\newtheorem{theorem}{Theorem}[section]
\newtheorem{assumption}[theorem]{Assumption}
\newtheorem{proposition}[theorem]{Proposition}

\newtheorem{lemma}[theorem]{Lemma}

\theoremstyle{remark}

\newtheorem{remark}{Remark}[section]

\def\E{\mathbb{E}}

\DeclareMathOperator*{\argmin}{arg\,min}

\begin{document}
%
\title{Regularized Weighted Chebyshev Approximations for Support Estimation}
%
%
%

\author{I~(Eli)~Chien,~\IEEEmembership{}
        and~Olgica~Milenkovic ~\IEEEmembership{}
        
\thanks{I~(Eli)~Chien and Olgica~Milenkovic are with ECE Department at University of Illinois Urbana-Champaign.}}
\maketitle

\begin{abstract}
We introduce a new method for estimating the support size of an unknown distribution which provably matches the performance bounds of the state-of-the-art techniques in the area and outperforms them in practice. In particular, we present both theoretical and computer simulation results that illustrate the utility and performance improvements of our method. The theoretical analysis relies on introducing a new weighted Chebyshev polynomial approximation method, jointly optimizing the bias and variance components of the risk, and combining the weighted minmax polynomial approximation method with discretized semi-infinite programming solvers. Such a setting allows for casting the estimation problem as a linear program (LP) with a small number of variables and constraints that may be solved as efficiently as the original Chebyshev approximation problem. Our technique is tested on synthetic data, textual data (Shakespeare's plays) and used to address an important problem in computational biology - estimating the number of bacterial genera in the human gut. On synthetic datasets, for practically relevant sample sizes, we observe significant improvements in the value of the worst-case risk compared to existing methods. The same is true of the text data. For the bioinformatics application, using metagenomic data from the NIH Human Gut and the American Gut Microbiome Projects, we generate a list of frequencies of bacterial taxa that allows us to estimate the number of bacterial genera to $\sim2300$.
\end{abstract}


%

\section{Introduction}
Estimating the support size of a discrete distribution is an important theoretical and data-driven problem~\cite{fisher1943relation,efron1976estimating}. In computer science, this task frequently arises in large-scale database mining and network monitoring where the objective is to estimate the types of database entries or IP addresses from a limited number of observations~\cite{raskhodnikova2009strong,bar2002counting,charikar2000towards}. In machine learning, support estimation is used to bound the number of clusters in clustering problems arising in semi-supervised or active learning~\cite{chien2018query,ashtiani2016clustering,chien2018hs}.
The most challenging practical support estimation issues arise in the ``small sample set'' regime in which one has only a limited number of observations for a large-support distribution. In such a setting, classical maximum likelihood frequency techniques are known to perform poorly~\cite{orlitsky2003always}. It is for this sampling regime that the estimation problem has received significant attention from both the theoretical computer science and machine learning community, as well as researchers from various computational data processing areas~\cite{acharya2011competitive,paninski2003estimation,bunge1993estimating,bar2001sampling,charikar2000towards,batu2000testing,nelson2012abundance,keinan2012recent}.

\subsection{Prior Work}

Approaches to distribution estimation in the small sample regime may be roughly grouped into two categories~\cite{valiant2013estimating,wu2019chebyshev,acharya2017unified,wu2018sample,pavlichin2017approximate,han2018local,yi2018data}. The first line of works~\cite{valiant2013estimating,acharya2017unified,pavlichin2017approximate} makes use of the maximum likelihood principle. While~\cite{acharya2017unified} constructs estimators based on the Profile Maximum Likelihood (PML)~\cite{orlitsky2004modeling}, the work reported in~\cite{valiant2013estimating} focuses on Sequence Maximum Likelihood (SML) estimators~\cite{aldrich1997ra}. The main advantage of ML-based methods is that they easily generalize to many other estimation tasks. For example, the authors of~\cite{acharya2017unified} showed that one and the same method may be used for entropy estimation, support coverage and distance to uniformity analysis. However, most ML-based estimators require large computational resources~\cite{pavlichin2017approximate,wu2019chebyshev}. To address the computational issue, a sophisticated approximate PML technique that reduces the computational complexity of support estimation at the expense of some performance loss was described in~\cite{pavlichin2017approximate}.
On the other hand, the second line of works~\cite{wu2019chebyshev,wu2018sample,han2018local,yi2018data} formulates support estimation as an approximation problem. The underlying methods, which we henceforth refer to as approximation-based methods, design estimators by minimizing the worst case risk. In particular,~\cite{wu2019chebyshev} uses shifted and scaled Chebyshev polynomials of the first kind to construct efficient estimators. In contrast, the authors of~\cite{yi2018data} suggest disposing with minmax estimators and introduce a data amplification technique with analytical performance guarantees. The aforementioned estimator is based on polynomial smoothing~\cite{orlitsky2016optimal} related to approximation techniques. Note that all described approximation-based estimators are computational efficient, with the exception of~\cite{han2018local}, as reported in~\cite{yi2018data}.

\subsection{Our contributions}

We adopt the approximation approach put forward in~\cite{wu2019chebyshev}, but significantly improve it in
practice by using \emph{weighted} polynomial approximation techniques largely unknown in the machine learning community~\cite{lubinsky2007survey}. The weighted approximation approach significantly improves the performance of known approximation methods while allowing one to seamlessly combine certain ideas already explored in different contexts for the support estimation problem. In particular, the technical novelty of the proposed method includes uniquely combining the following approaches:\\
\emph{Using weighted Chebyshev approximations:} The authors of~\cite{wu2019chebyshev} used classical Chebyshev approximations to construct their estimator, while ignoring the exponential weighting term that arises due to Poissonization. As will be shown in subsequent analyses, the exponential weights play a major role in improving the performance of our method as well as making it computationally tractable through a new take on the ``localization idea''.\\
\emph{Using the variance as a regularizer:} The idea of jointly optimizing the bias and variance of entropy estimators has been previously reported in~\cite{paninski2003estimation}, but only within the framework of ML estimation. The authors of~\cite{wu2019chebyshev} focused on optimizing the bias, but accounted for the variance through a separate parameter tuning technique. To the best of our knowledge, \emph{jointly optimizing the bias and variance terms in approximation-based methods} is novel.\\
\emph{Using semi-infinite programming (SIP) techniques with discretization:} SIP techniques have also been explored in the approximation-based support estimation literature~\cite{han2018local}, but for a different objective and without rigorously establishing the convergence results of the discretization technique. We show that the solution of our discretized SIP for the regularized weighted Chebyshev approximation problem converges to the true unique optimal solution.

The theoretical and practical results presented are as follows. First, we cast the estimation problem in terms of a weighted and regularized Chebyshev approximation~\cite{mason2002chebyshev} problem for a normalized quadratic loss function, in which the normalization term is an upper bound on the support~\cite{wu2019chebyshev}. The regularized weighted Chebyshev estimator provably offers \emph{worst case risk bounds} that match or improve those of the estimator in~\cite{wu2019chebyshev}. Although the risk bound of the aforementioned estimator is order-optimal,
significant improvements in the risk exponent are possible and demonstrated in the experimental section of our work.
It is worth pointing out that \emph{adding weights} into the approximation formulation is of crucial importance in reducing the error in the bias. We solve the regularized weighted Chebyshev approximation problem via its epigraph formulation, which takes the form of a semi-infinite program (SIP)~\cite{lopez2007semi}. We prove that the underlying SIP can be solved consistently and efficiently through discretization, resulting in a small Linear Program (LP) of size \emph{decreasing with the number of samples.} Interestingly, despite the fact that we also use an LP solver
as was done in one of the best performing ML-based approaches~\cite{valiant2013estimating}, the ML-LP formulation has a number of variables and constraints that increases with the number of samples. Our experimental results also reveal that the regularized weighted Chebyshev estimator significantly outperforms the low-complexity ML estimator described in~\cite{pavlichin2017approximate}, the benchmark approximation algorithm~\cite{wu2019chebyshev}, as well as the recently proposed sample amplification estimator~\cite{yi2018data}. The running times of all described algorithms are comparable on moderate-to-large sample set sizes ($\sim 10^6$).

Additional simulations also account for the fact that the normalized square loss causes the worst-case distribution to be close to uniform, resulting in performance issues previously described in~\cite{yi2018data}. To address this problem, we use a different normalization term in the squared loss equal to the actual support size.
We show that the weighted regularized Chebyshev estimator that minimizes the worst case risk under the modified loss outperform all methods mentioned above and performs consistently well on many tested distributions.

The paper is organized as follows. In Section~\ref{sec:Probform}, we introduce the relevant notation and the support estimation problem, and describe a class of estimators termed \emph{polynomial class estimators}. The same section also provides a brief review of relevant results from~\cite{wu2019chebyshev}. In Section ~\ref{sec_bias}, we outline the bias analysis for polynomial class estimators and describe the additional technical challenges one needs to overcome when dealing with weighted minmax polynomial approximations. The same section contains our worst-case risk analysis. Section~\ref{sec:Simulations} is devoted to experimental verifications and testing, both on synthetically generated data and metagenomic data samples from the NIH Human Microbiome~\cite{peterson2009nih} and American Gut Microbiome project (http://americangut.org/publications/).

\section{Problem formulation and new polynomial class estimators}\label{sec:Probform}

Let $P = (p_1,p_2,\ldots)$ be a discrete distribution over some countable alphabet and let $X_1,\ldots,X_n$ be i.i.d. samples drawn according to the distribution $P$. The problem of interest is to estimate the support size, defined as $S(P) = \sum_{i}\mathbf{1}_{\{p_i>0\}}$. Henceforth,  we use $S$ instead of $S(P)$ to avoid notational clutter. We make the assumption that the minimum non-zero probability of the distribution $P$ is greater than $\frac{1}{k},$ for some $k\in \mathbb{R}^{+}$, i.e., $\inf\{p\in P\,| \, p>0\}\geq \frac{1}{k}$. Furthermore, we let $D_k$ denote the space of all probability distribution satisfying $\inf\{p\in P\, |\, p>0\}\geq \frac{1}{k}$. Clearly, $S \leq k$, $\forall P\in D_k$. A sufficient statistics for $X_1,\ldots,X_n$ is the empirical distribution (i.e., histogram) $N = (N_1,N_2,\ldots),$ where $N_i = \sum_{j=1}^{n}\mathbf{1}_{\{X_j=i\}}$ and $\mathbf{1}_A$ stands for the indicator function of the event $A$.

The focal point of our analysis is $R^{\ast}(k,n)$, the minmax risk under normalized squared loss,
\begin{equation} \label{eq:minimax}
    R^{\ast}(k,n) = \inf_{\hat{S}}\sup_{P\in D_k} \mathbb{E}\left[\left(\frac{\hat{S}(N)-S}{k}\right)^2\right].
\end{equation}
We seek a support estimator $\hat{S}$ that minimizes
\begin{align*}
&\sup_{P\in D_k} \mathbb{E}\left[\left(\frac{\hat{S}(N)-S}{k}\right)^2\right] \\
&= \sup_{P\in D_k} \left[ \mathbb{E}^2\left(\frac{\hat{S}(N)-S}{k}\right)+var\left(\frac{\hat{S}(N)-S}{k}\right) \right]. \notag
\end{align*}
The first term within the supremum captures the expected bias of the estimator $\hat{S}$. The second term represents the variance of the estimator $\hat{S}$. Hence, ``good'' estimators are required to balance out the worst-case contributions of
the bias and variance.

The Chebyshev polynomial of the first kind of degree $L$ is defined as
$T_L(x) = \cos(L \arccos(x)) = \frac{z^L+z^{-L}}{2},$ where $z$ is the solution of the quadratic equation $z+z^{-1} = 2x$. From the definition, it is easy to see that the polynomial $T_L$ is bounded in the interval $[-1,1]$. Chebyshev polynomials may be scaled and shifted to lie in an interval $[l,r]$ not necessarily equal to $[-1,1]$,
\begin{equation*}
    R_L(x) = -\frac{T_L(\frac{2x-r-l}{r-l})}{T_L(\frac{-r-l}{r-l})} \triangleq \sum_{j=0}^{L}\tilde{a}_{j}x^j.
\end{equation*}
Clearly, $R_L(0) = -1,$ and $\tilde{a}_{0} = -1$. The coefficients in the above expansion equal
\begin{equation}\label{coeff:wu}
    \tilde{a}_{j} = \frac{R_L^{(j)}(0)}{j!}. 
\end{equation}
The estimator used in~\cite{wu2019chebyshev} takes the form $\tilde{S}=\sum_{i}\tilde{g}_{L}(N_i),$ where
\begin{align}
&\tilde{g}_{L}(j) =
    \begin{cases}
          \tilde{a}_{j} j!+1,  &\mbox{ if } j\leq L,\\
          1,         &\mbox{ otherwise, }
    \end{cases},\nonumber\\
    &L \triangleq \lfloor c_0\log k \rfloor,\; [l,r] \triangleq \left[\frac{n}{k},c_1 \log k\right].
\end{align}
The estimator $\tilde{S}$ is order-optimal \emph{in the exponent} under the risk~(\ref{eq:minimax}). Nevertheless, the estimator $\tilde{S}$ is designed to minimize only a bias term that ignores a multiplicative exponential and for a
given pair of parameters $(c_0,c_1)$. The parameters $(c_0,c_1)$ are set to $c_0 \approx 0.558$ and $c_1 = 0.5$ in order to ensure provably good balance between the bias and variance of the estimator. The performance guarantees of the estimator are stated in the theorem that follows. 
\begin{theorem}[Theorem 1 in~\cite{wu2019chebyshev}]\label{thm:wu}
    For all $k,n\geq 2$, the optimal achievable risk~\eqref{eq:minimax} is bounded as
    \begin{equation}
        R^{\ast}(k,n) = \exp(-\Theta(\sqrt{\frac{n\log k}{k}}\vee \frac{n}{k} \vee 1))
    \end{equation}
    In addition, if $\frac{k}{\log k} \ll n \ll k\log k$, as $k\rightarrow \infty$,
    \begin{align*}
        &\exp(-(3.844+o(1))\sqrt{\frac{n\log k}{k}}) \leq R^{\ast}(k,n)\\
        &\leq \exp(-(1.579+o(1))\sqrt{\frac{n\log k}{k}}).
    \end{align*}
\end{theorem}
The estimator $\tilde{S}$ attains the upper bound, which has a constant negative exponent equal to one half of the
one corresponding to the lower bound. Consequently, $\tilde{S}$ may be improved by accounting for the omitted exponential weights and jointly optimizing the bias and variance term in the squared loss. This is accomplished by our proposed estimator that \emph{directly} operates on the weighted squared loss.

To this end, we define an extended class of polynomial based estimators as follows. Given a parameter $L \in \mathbb{N}$, we say that an estimator $\hat{S}$ is a polynomial class estimator with parameter $L$ (i.e.,  a $Poly(L)$ estimator) if it takes the form $\hat{S} = \sum_{i}g_L(N_i),$ where $g_L$ is defined as
\begin{equation}
g_L(j) =
    \begin{cases}
          a_j j!+1,  &\mbox{if } j\leq L\\
          1,         &\mbox{ otherwise. }
    \end{cases}
\end{equation}
Here, $a_j \in \mathbb{R},$ and $a_0 = -1,$ since this choice ensures that $g_L(0) = 0$. One can associate an estimator $\hat{S}$ with its corresponding coefficients $\mathbf{a}$, and define a family of estimators
\begin{equation*}
    Poly(L) = \bigg\{\mathbf{a}\in \mathbb{R}^{L+1}|a_0 = -1\bigg\}.
\end{equation*}
Clearly, $\tilde{S} \in Poly(L)$, with corresponding coefficients
\begin{equation}\label{biaseq5}
   \mathbf{\tilde{a}} = \argmin_{\mathbf{a}\in Poly(L)}\, \sup_{\lambda\in [l=\frac{n}{k},r=c_1\log k]} | \sum_{l=0}^{L} a_l \lambda_i^l  |.
\end{equation}
Next, we show that the original minmax problem can be rewritten as a regularized exponentially weighted Chebyshev approximation problem~\cite{lubinsky2007survey}. Once a proper interval is identified, it is possible to efficiently determine the best coefficients $\mathbf{a}$ in the class $Poly(L)$.


\section{Estimator analysis}\label{sec_bias}
In order to jointly optimize the bias and variance term in the squared loss, we start by analyzing $\sup_{P\in D_k}\E \left( \frac{S-\hat{S}}{k}\right)^2$ directly. Classical Poissonization arguments~\cite{wu2019chebyshev} lead to
\begin{align*}
    &\mathbb{E}\left(\frac{S-\hat{S}}{k}\right)^2 = \frac{1}{k^2}\bigg\{\sum_{i:\lambda_i>0}\bigg(\sum_{l=0}^{L}e^{-\lambda_i}a_l^2\lambda_i^ll!\bigg)\\
    &+\sum_{i\neq j:\lambda_i\lambda_j>0}\bigg(e^{-\lambda_i}P_L(\lambda_i,\mathbf{a})\bigg)\bigg(e^{-\lambda_j}P_L(\lambda_j,\mathbf{a})\bigg)\bigg\},
\end{align*}
where $P_L(\lambda,\mathbf{a}) \triangleq \sum_{l=0}^{L}a_l\lambda^l$. Taking the supremum over $D_k$ we can bound the risk as

\begin{align*}
    &\leq \sup_{\mathbf{\lambda}:\lambda_i\in [\frac{n}{k}, n]} \frac{1}{k^2}\bigg\{\sum_{i:\lambda_i>0}\bigg(\sum_{l=0}^{L}e^{-\lambda_i}a_l^2\lambda_i^ll!\bigg)\\
    &+\sum_{i\neq j:\lambda_i\lambda_j>0}\bigg(e^{-\lambda_i}P_L(\lambda_i,\bold{a})\bigg)\bigg(e^{-\lambda_j}P_L(\lambda_j,\mathbf{a})\bigg)\bigg\}\\
    & \leq \sup_{\lambda\in [\frac{n}{k}, n]}\bigg\{\frac{1}{k}\bigg(\sum_{l=0}^{L}e^{-\lambda}a_l^2\lambda^ll!\bigg)+\bigg(e^{-\lambda}P_L(\lambda,\mathbf{a})\bigg)^2\bigg\}.
\end{align*}

In the above inequality, we used the fact that $S\leq k$ and $\bigg(\sum_{l=0}^{L}e^{-\lambda}a_l^2\lambda^ll!\bigg)>0,$ for all $\lambda>0$.
Hence, the optimization problem at hand reads as
\small
\begin{align}\label{sqlosseq1}
 \inf_{\mathbf{a}\in Poly(L)}\sup_{\lambda\in [\frac{n}{k}, n]}\bigg\{\frac{1}{k}\bigg(\sum_{l=0}^{L}e^{-\lambda}a_l^2\lambda^ll!\bigg)+\bigg(e^{-\lambda}P_L(\lambda,\mathbf{a})\bigg)^2\bigg\}.
\end{align}
\normalsize
Problem~\eqref{sqlosseq1} represents a regularized weighted Chebyshev approximation problem. If we ignore the first term in~\eqref{sqlosseq1}, the optimization problem becomes
\begin{align*}
    \inf_{\mathbf{a}\in Poly(L)}\sup_{\lambda\in [\frac{n}{k}, n]}\bigg(e^{-\lambda}P_L(\lambda,\mathbf{a})\bigg)^2.
\end{align*}
The term $e^{-\lambda}P_L(\lambda,\mathbf{a})$ corresponds to the bias of the estimator. It is straightforward to see that the optimal choice of $\mathbf{a}$ for the above problem is a solution of
\begin{align}\label{eq:WCapprox}
    \inf_{\mathbf{a}\in Poly(L)}\sup_{\lambda\in [\frac{n}{k}, n]}\bigg |e^{-\lambda}P_L(\lambda,\mathbf{a})\bigg|.
\end{align}
Problem~\eqref{eq:WCapprox} is an exponential weighted Chebyshev approximation problem~\cite{mason2002chebyshev}.

The first term $\frac{1}{k}\bigg(\sum_{l=0}^{L}e^{-\lambda}a_l^2\lambda^ll!\bigg)$, which corresponds to the variance, may be written as
\begin{align*}
    & \frac{1}{k}\bigg(\sum_{l=0}^{L}e^{-\lambda}a_l^2\lambda^ll!\bigg) = \mathbf{a}^T\mathbf{M}(\lambda)\mathbf{a} \triangleq ||\mathbf{a}||_{\mathbf{M}(\lambda)}^2,\\
    &\mathbf{M}(\lambda) \triangleq \frac{e^{-\lambda}}{k} Diag(\lambda^{0}0!,\lambda^{1}1!,...,\lambda^{L}L!).
\end{align*}
Clearly, $||.||_{\mathbf{M}(\lambda)}$ is a valid norm, and consequently, the first term in~\eqref{sqlosseq1} can be viewed as a regularizer.

Note that both problems~\eqref{biaseq5} and~\eqref{eq:WCapprox} aim to minimize the bias term. However, simple algebra reveals
\small
\begin{align}
    \sup_{P\in D_k}\frac{1}{k}|\mathbb{E}(S-\hat{S}(N))|
     &\leq \sup_{\lambda\in [\frac{n}{k},n]}|e^{-\lambda}P_L(\lambda,\bold{a})|\label{biaseq3.1}\\
    &\leq e^{-\frac{n}{k}}\sup_{\lambda\in [\frac{n}{k},n]}|P_L(\lambda,\bold{a})|.\label{biaseq3.2}
\end{align}
\normalsize
where \eqref{biaseq3.1} is equivalent to~\eqref{eq:WCapprox}, while~\eqref{biaseq3.2} resembles problem~\eqref{biaseq5}, except for a different optimization interval used within the supremum. Optimizing~\eqref{biaseq3.1} instead of~\eqref{biaseq3.2} should give us a smaller bias as the exponential weight is inherent to the formulation. Note that the authors of~\cite{wu2019chebyshev} choose a shorter interval in~\eqref{biaseq5} in order to decrease the contribution of the variance to the loss. The modified bound in~\eqref{biaseq3.2} was minimized with respect to the coefficients $\textbf{a}$, using the minmax property of Chebyshev polynomials~\cite{timan2014theory,mason2002chebyshev}, resulting in $\tilde{\textbf{a}}$.

Let us now turn back to the solution of the minmax problem~\eqref{sqlosseq1}, denoted by $\mathbf{a}^{\star}$.
Clearly, using $\mathbf{a}^{\star}$ instead of $\tilde{\mathbf{a}}$ is guaranteed to reduce the reported \emph{upper bound} on the risk since $\mathbf{a}^{\star}$ \emph{jointly} minimizes both contributing terms.
\begin{proposition}\label{prop:beatwu}
    Let $\hat{S}^{\star}$ be the estimator associate with the coefficients $\mathbf{a}^{\star}$ optimizing~\eqref{sqlosseq1}. Whenever $\frac{k}{\log k} \ll n \ll k\log k$, and $k\rightarrow \infty$, the estimator $\hat{S}^{\star}$ has a smaller worst-case risk bound than $\tilde{S}$.
\end{proposition}
\begin{proof}
    The proof follows directly from the fact that the upper bound of Theorem~\ref{thm:wu} is derived by upper bounding~\eqref{sqlosseq1}, plugging $\tilde{\mathbf{a}}$ into the expression, and replacing $e^{-\lambda}$ by $e^{-\frac{n}{k}}$.
\end{proof}
We remark that improving upper bounds does not establish improvements in the actual worst-case risk. Nevertheless, whenever faced with complicated analytical settings, comparisons of bounds may provide useful insights, and in our case, these insights are supported by strong simulation evidence.
\subsection{Solving problem~\eqref{sqlosseq1}}

To directly solve the optimization problem of interest, we turn to weighted Chebyshev approximations~\cite{mason2002chebyshev} and semi-infinite programming. Solving for $\mathbf{a}^{\star}$ directly appears to be difficult, so we instead resort to numerically solving the epigraph formulation of problem~\eqref{sqlosseq1} and proving that the numerical solution is asymptotically consistent.

The epigraph formulation of~\eqref{sqlosseq1} is of the form (\cite{boyd2004convex}, Chapter 6.1)
\begin{equation}\label{sqlosseq2}
    \begin{split}
        &\min_{t,a_1,...,a_L} t \;\;\;\text{subject to}\\
        & \bigg\{\frac{1}{k}\bigg(\sum_{l=0}^{L}e^{-\lambda}a_l^2\lambda^ll!\bigg)+\bigg(e^{-\lambda}P_L(\lambda,\mathbf{a})\bigg)^2\bigg\}\leq t,\\
        &\forall \lambda\in [\frac{n}{k}, n], \text{with }a_0=-1.
    \end{split}
\end{equation}
Note that~\eqref{sqlosseq2} is a semi-infinite programming problem. There are many algorithms that can be used to numerically solve~\eqref{sqlosseq2}, such as the discretization method, and the central cutting plane, KKT reduction and SQP reduction methods ~\cite{lopez2007semi,reemtsen1998numerical}. For simplicity, we focus on the discretization method. For this purpose, we first form a grid of the interval $[\frac{n}{k}, n]$ involving $s$ points, denoted by $\text{Grid}([\frac{n}{k}, n],s)$. Problem~\eqref{sqlosseq2} may be viewed as an LP with infinitely many quadratic constraints, which is not solvable. Hence, instead of addressing~\eqref{sqlosseq2}, we focus on solving the relaxed problem
\begin{equation}\label{sqlosseq3}
    \begin{split}
        &\min_{t,a_1,...,a_L} t \;\;\;\text{subject to}\\
        & \bigg\{\frac{1}{k}\bigg(\sum_{l=0}^{L}e^{-\lambda}a_l^2\lambda^ll!\bigg)+\bigg(e^{-\lambda}P_L(\lambda,\mathbf{a})\bigg)^2\bigg\}\leq t,\\
        &\forall \lambda\in \text{Grid}([\frac{n}{k}, n],s), \text{with }a_0=-1.
    \end{split}
\end{equation}
As will be discussed in greater detail in Section~\ref{subsec:discmethod}, the solution of the relaxed problem is asymptotically consistent with the solution of the original problem (i.e., as $s$ goes to infinity, the optimal values of the objectives of the original and relaxed problem are equal). Problem~\eqref{sqlosseq3} is an LP with a finite number of quadratic constraints that may be solved using standard optimization tools. Unfortunately, the number of constraints scales with the length of the grid interval, which in the case of interest is linear in $n$. This is an undesired feature of the approach, but it may be mitigated through the following theorem which demonstrates that an optimal solution of the problem may be found over an interval of length proportional to the significantly smaller value $\log\,k$, where $\frac{k}{\log\,k} \lesssim n$ is the fundamental bound for support estimation. We relegate the proof to the Appendix.
\begin{theorem}\label{lma:interval_sq}
    For any $\mathbf{a}\in Poly(L)$ and $L=\lfloor c_0 \, \log\,k \rfloor$, and $c_0 = 0.558$, let
    $$g(\mathbf{a},\lambda) = \frac{1}{k}\bigg(\sum_{l=0}^{L}e^{-\lambda}a_l^2\lambda^ll!\bigg)+\bigg(e^{-\lambda}P_L(\lambda,\mathbf{a})\bigg)^2.$$
    Then, we have
    \small
    \begin{align*}
        \sup_{\lambda\in [\frac{n}{k}, n]}g(\mathbf{a},\lambda)=
        \begin{cases}
            \sup_{\lambda\in [\frac{n}{k}, 6.5L]}g(\mathbf{a},\lambda) & \text{if }\frac{n}{k}\leq 6.5L\\
            g(\mathbf{a},\frac{n}{k}) & \text{if }\frac{n}{k}> 6.5L.
        \end{cases}
    \end{align*}
    \normalsize
\end{theorem}

\begin{remark}\label{remark:pi_interval}
Using exponential weighted approximation theory~\cite{lubinsky2007survey} (Theorem~\ref{thm:MRS} in the Appendix), one may shrink the optimization interval to $[\frac{n}{k},\frac{\pi}{2}L + \frac{n}{k}]$ if the regularization term is omitted. However, it remains an open problem to extend the approach of~\cite{lubinsky2007survey} used in our proof to account for more general weighted sums of polynomials.
\end{remark}

Consequently, the optimization problem of interest equals
\begin{equation}\label{sqlosseq5}
    \begin{split}
        &\min_{t,a_1,...,a_L} t \;\;\;\text{subject to}\\
        & \bigg\{\frac{1}{k}\bigg(\sum_{l=0}^{L}e^{-\lambda}a_l^2\lambda^ll!\bigg)+\bigg(e^{-\lambda}P_L(\lambda,\mathbf{a})\bigg)^2\bigg\}\leq t,\\
        &\forall \lambda\in \text{Grid}([\frac{n}{k}, 6.5L],s), \text{with }a_0=-1.
    \end{split}
\end{equation}
Since $L=\lfloor c_0\, \log\, k \rfloor$, the length of the optimization interval in~\eqref{sqlosseq5} is proportional to $\log\,k$.
\subsection{Convergence of the discretization method}\label{subsec:discmethod}

It seems intuitive to assume that as $s$ grows, the solution of the relaxed semi-infinite program approaches the optimal solution of the original problem~\eqref{sqlosseq2}. This intuition can be rigorously justified for the case of objective functions and constraints that are ``well-behaved'', as defined in~\cite{reemtsen1991discretization} and~\cite{still2001discretization}. The first line of work describes the conditions needed for convergence, while the second establishes the convergence rate given that the discretized solver converges. We use these results in conjunction with a number of properties of our objective SIP to establish the claim in the following theorem. Once again, the proof is delegated to the Appendix.
\begin{theorem}\label{thm:discretization}
    Let $s$ be the number of uniformly placed grid points on the interval~\eqref{sqlosseq5}, and let $d\triangleq \frac{6.5L-\frac{n}{k}}{s-1}$ be the length of the discretization interval. As $d \rightarrow 0$, the optimal objective value $t_d$ of the discretized SIP~\eqref{sqlosseq5} converges to the optimal objective value of the original SIP $t^\star$. Moreover, the optimal solution is unique. The convergence rate of $t_d$ to $t^\star$ equals $O(d^2)$. If the optimal solution of the SIP is a strict minimum of order one (i.e., if $t-t^\star \geq C||\mathbf{a}-\mathbf{a}^\star||$ for some constant $C>0$ and for all feasible neighborhoods of $\mathbf{a}^\star$), then the solution of the discretized SIP also converges to an optimal solution with rate $O(d^2)$.
\end{theorem}

\section{Simulations}\label{sec:Simulations}
Next, we compare our estimator, referred to as the Regularized Weighted Chebyshev (RWC) method, with the Good-Turing (GT) estimator, the WY estimator of~\cite{wu2019chebyshev}, the PJW estimator described in~\cite{pavlichin2017approximate} and the HOSW estimator of~\cite{yi2018data}. We do not compare our method with the estimators introduced in~\cite{valiant2013estimating,han2018local} due to their high computational complexity~\cite{wu2019chebyshev,yi2018data}.


In the first experiment, we evaluate the maximum risk with normalized squared loss of all listed estimators over six different distributions: The uniform distribution with $p_i = \frac{1}{k}$, the Zipf distributions with $p_i\propto i^{-\alpha}$, and $\alpha$ equal to $1.5$, $1$, $0.5$ or $0.25$, and the Benford distribution with $p_i \propto \log(i+1)-\log(i)$. We choose the support sizes for the Zipf and Benford distribution so that the minimum non-zero probability mass is roughly $10^{-6}$. We ran the estimator $100$ times to calculate the risk. For both approximation-based estimators, we fix $c_0$ to be $0.558$. With our proposed method, we solved~\eqref{sqlosseq5} on a grid with $s=1000$ points on the proposed interval $[\frac{n}{k},6.5L]$. For the estimator described in~\cite{wu2019chebyshev}, we set $c_1 = 0.5$ according to the recommendation made in the cited paper. The GT method used for comparison first estimates the total probability of seen symbols (e.g., sample coverage) according to $\hat{C} = 1-\frac{h_1}{n},$ and then estimates the support size according to $\hat{S}_{GT} = \frac{\hat{S}_{\text{c}}}{\hat{C}}$; here, $\hat{S}_{\text{c}}$ stands the (naive) counting estimator. Note that $h_1$ equals the number of different alphabet symbols observed only once in the $n$ samples.

Figure~\ref{fig:maxresult} shows that the RWC estimator has a significantly better worst case performance compared to all other methods when tested on the above described collection of distributions, provided that $n\geq 0.2k$. Also, both RWC and WY estimators have significantly better error exponents compared to the GT, PJW and HOSW estimators. The GT and PJW estimators perform better than RWC if $n \lesssim \frac{k}{\log k},$ which confirms the results of Theorem~\ref{thm:wu} and of our Proposition~\ref{prop:beatwu}.

In the second set of experiments, we change the normalization from $(1/k)^2$ to $(1/S)^2$ as was also done in~\cite{yi2018data}. The RWC-S estimator minimizes an upper bound on the worst case risk $\E\left(\frac{\hat{S}-S}{S}\right)^2$. A detailed description of this algorithm and intuitive explanation why it may outperform the RWC method is relegated to the Appendix. Figures~\ref{fig:maxresult_S} (i)-(l) illustrate that our RWC-S estimator significantly outperforms all other estimators with respect to the worst case risk. Moreover, the RWC-S estimator also outperforms all known estimators on almost all tested distributions.

\begin{figure*}[!htb]
  \centering
  \subfigure[Weighted vs classical Chebyshev approximation. \label{fig:BiasforexpW}]{\includegraphics[width=0.245\linewidth]{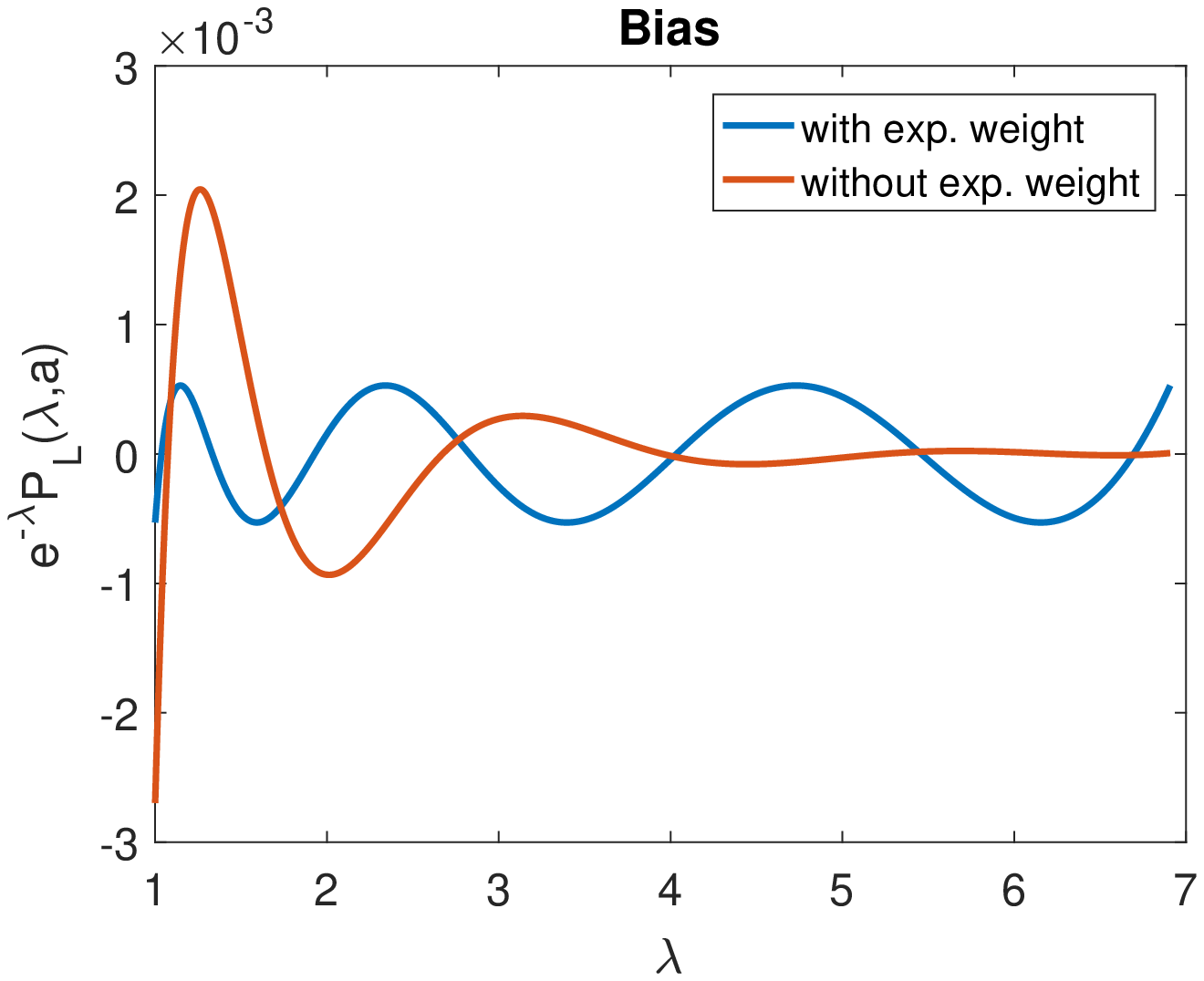}}
  \subfigure[The coefficients of $g_L$. \label{fig:gLcompare}]{\includegraphics[width=0.245\linewidth]{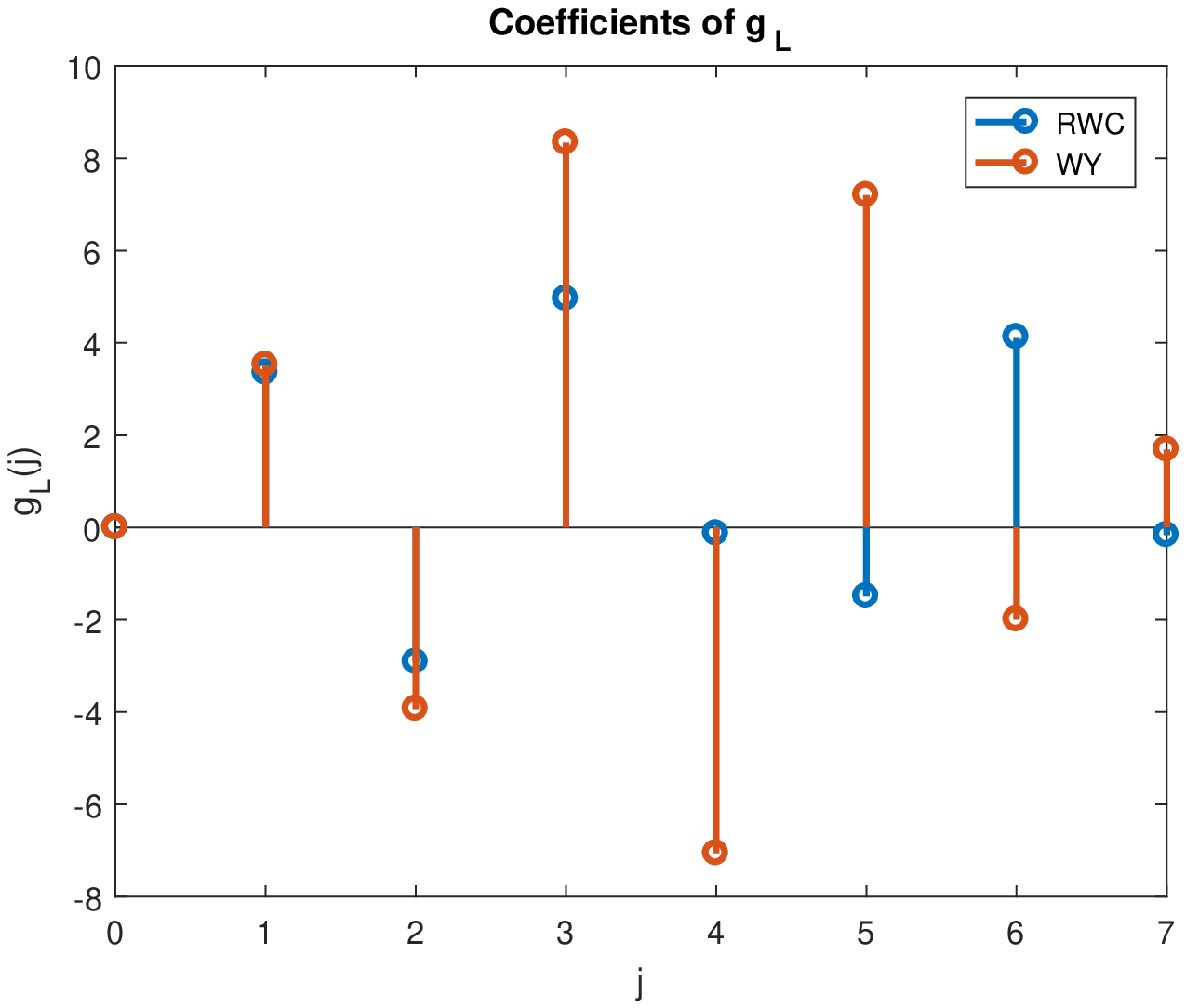}}
  \subfigure[Worst case MSE$/k^2$. \label{fig:maxresult}]{\includegraphics[width=0.245\linewidth]{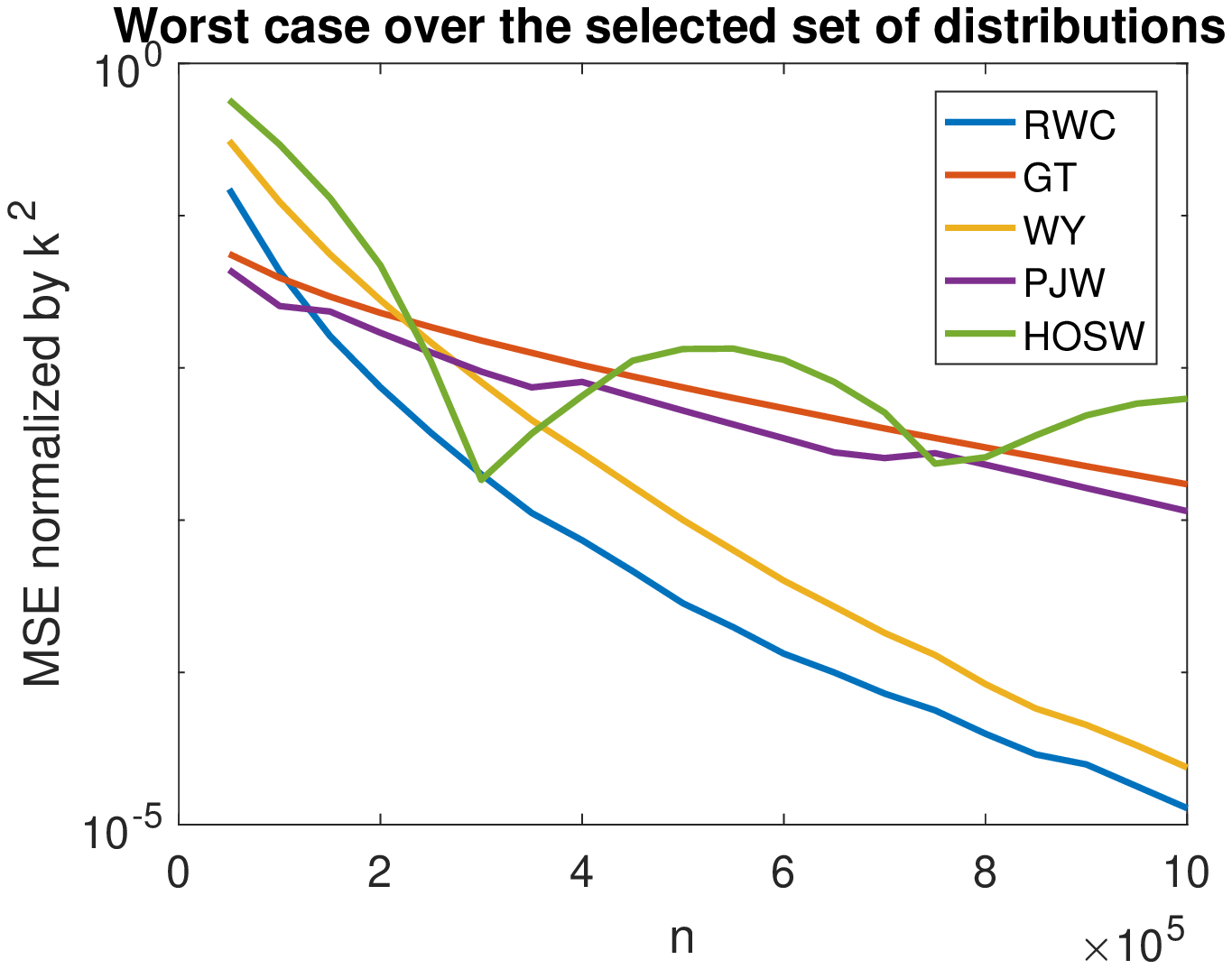}}
  \subfigure[Worst case MSE$/S^2$. \label{fig:maxresult_S}]{\includegraphics[width=0.245\linewidth]{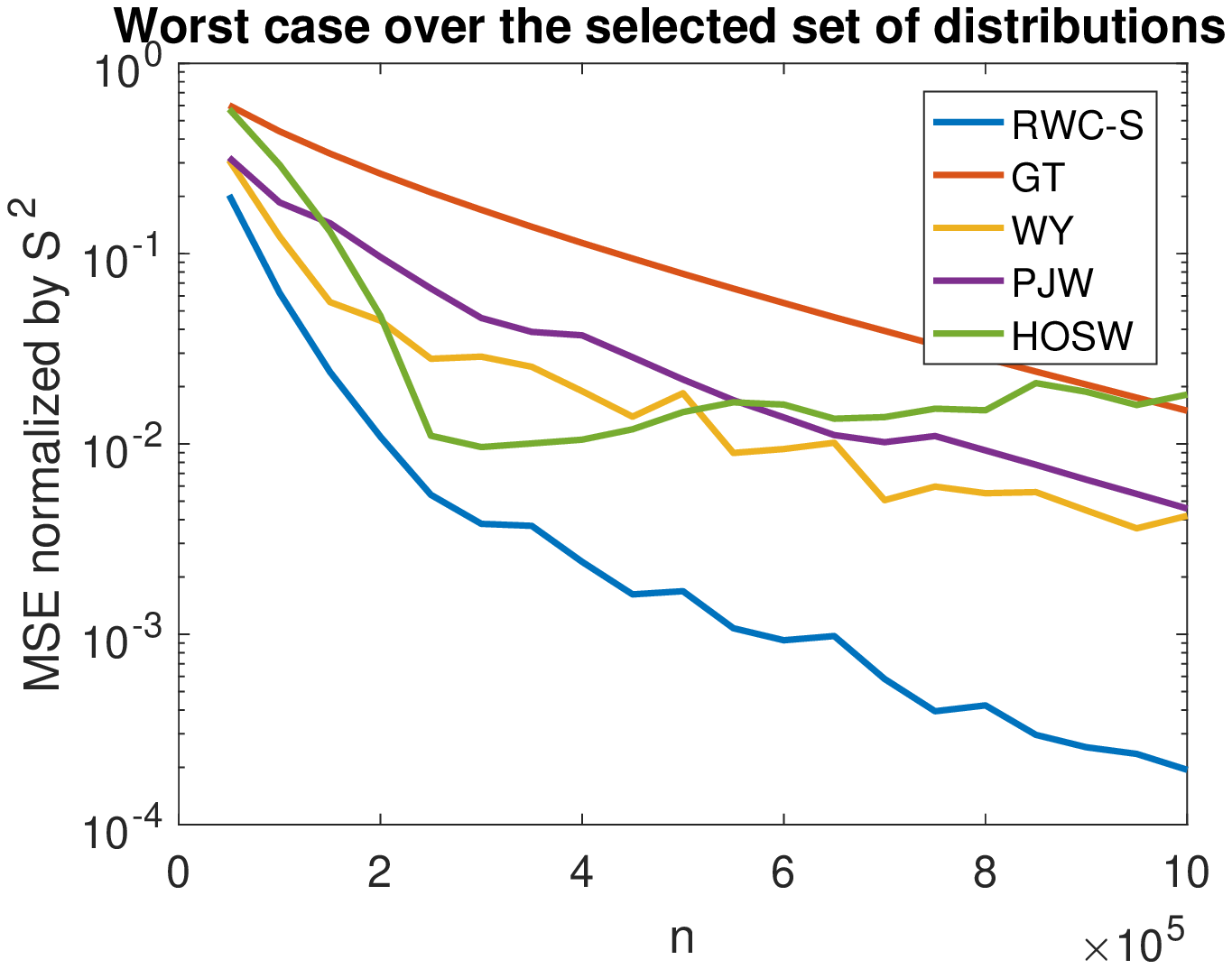}}
  \subfigure[Uniform distribution.]{\includegraphics[width=0.245\linewidth]{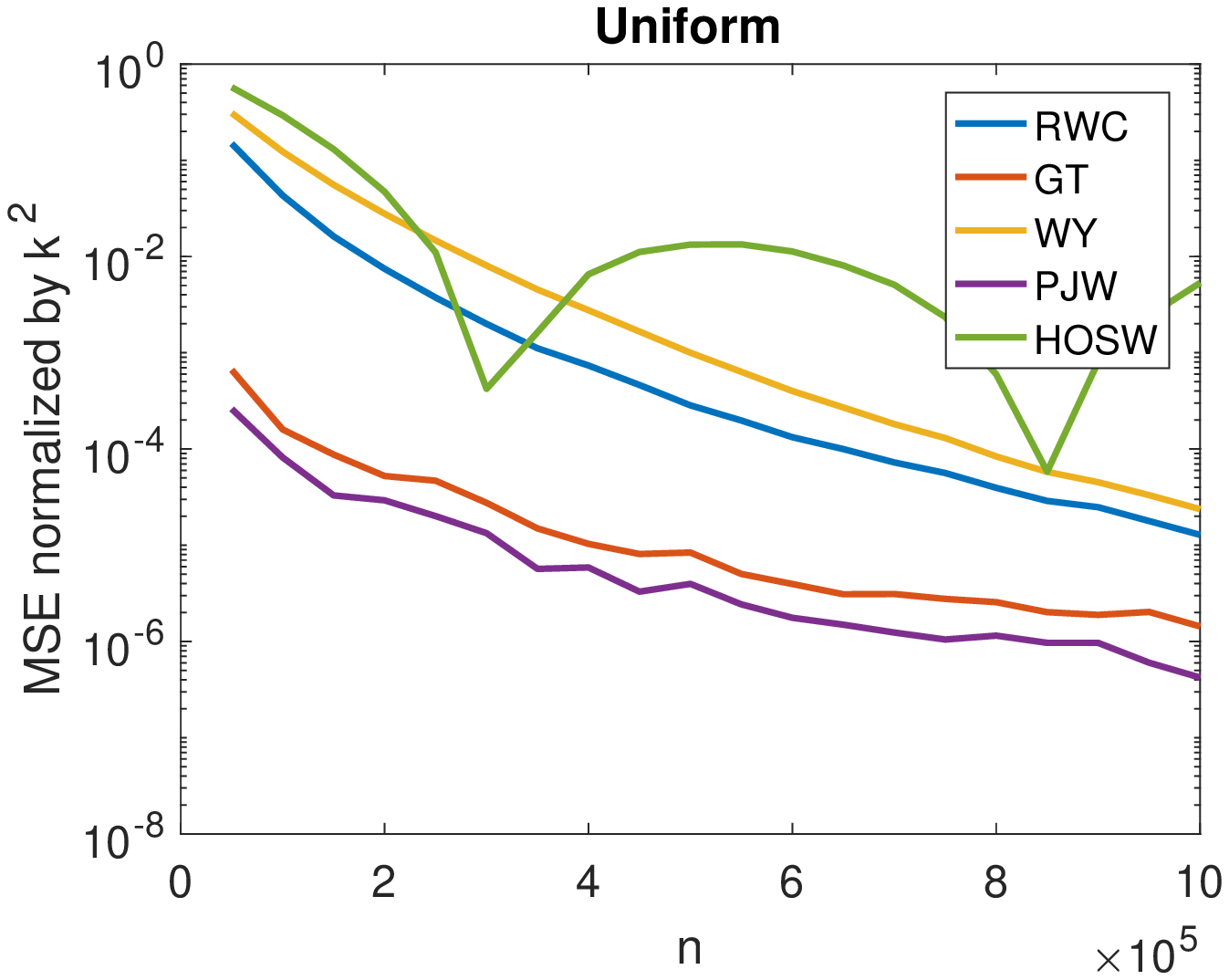}}
  \subfigure[Benford distribution.]{\includegraphics[width=0.245\linewidth]{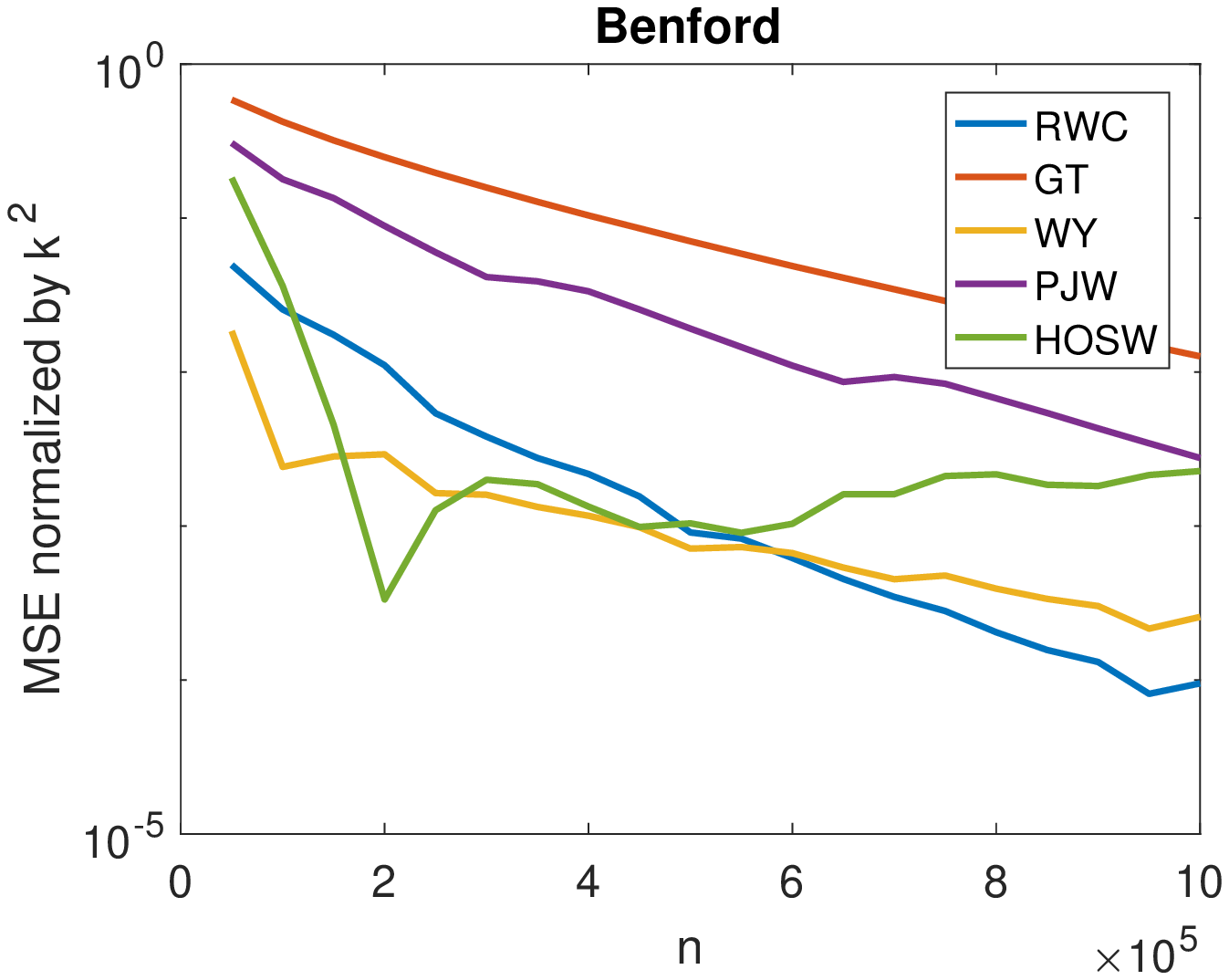}}
  \subfigure[Zipf$(1.5)$ distribution.]{\includegraphics[width=0.245\linewidth]{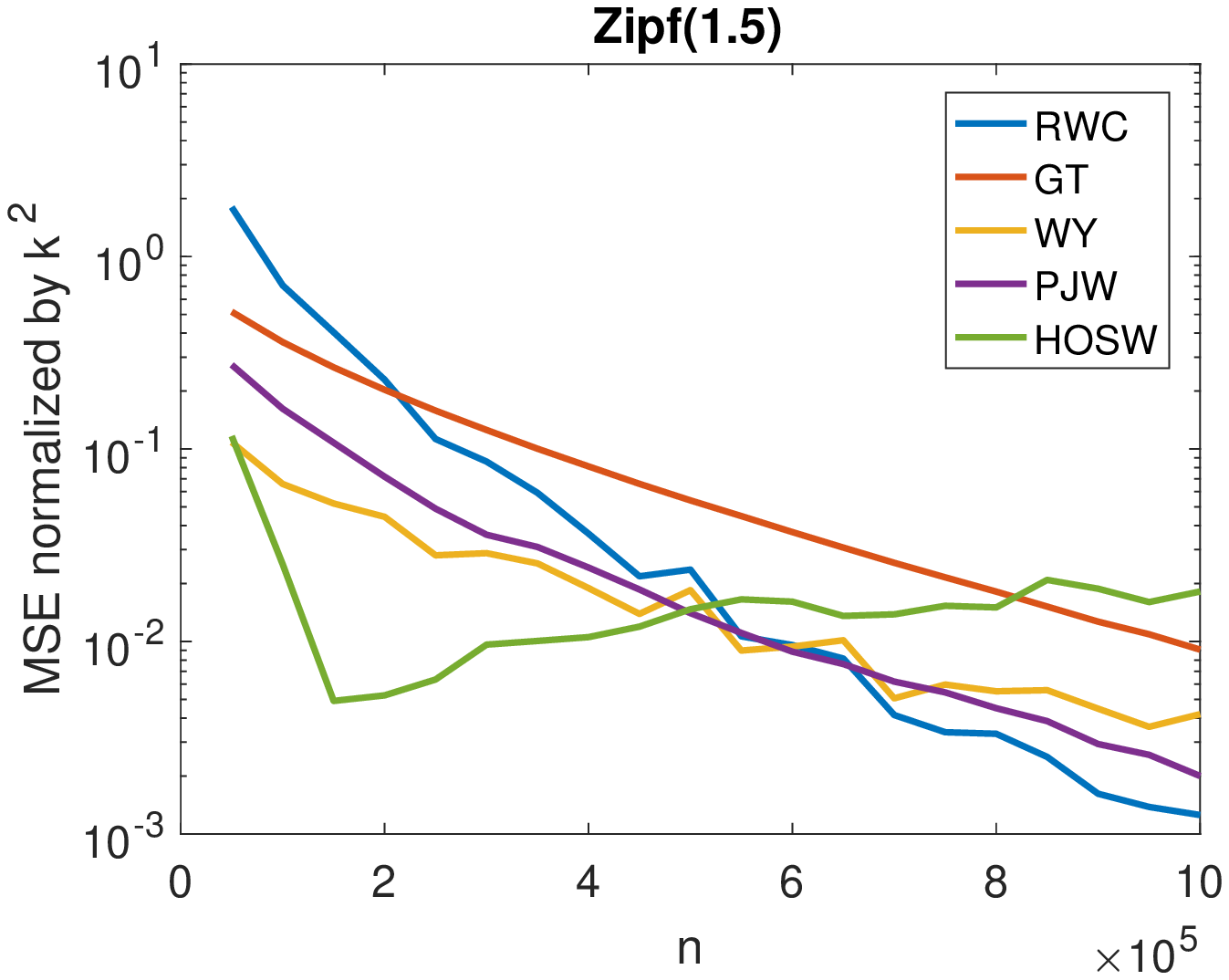}}
  \subfigure[Zipf$(0.25)$ distribution.]{\includegraphics[width=0.245\linewidth]{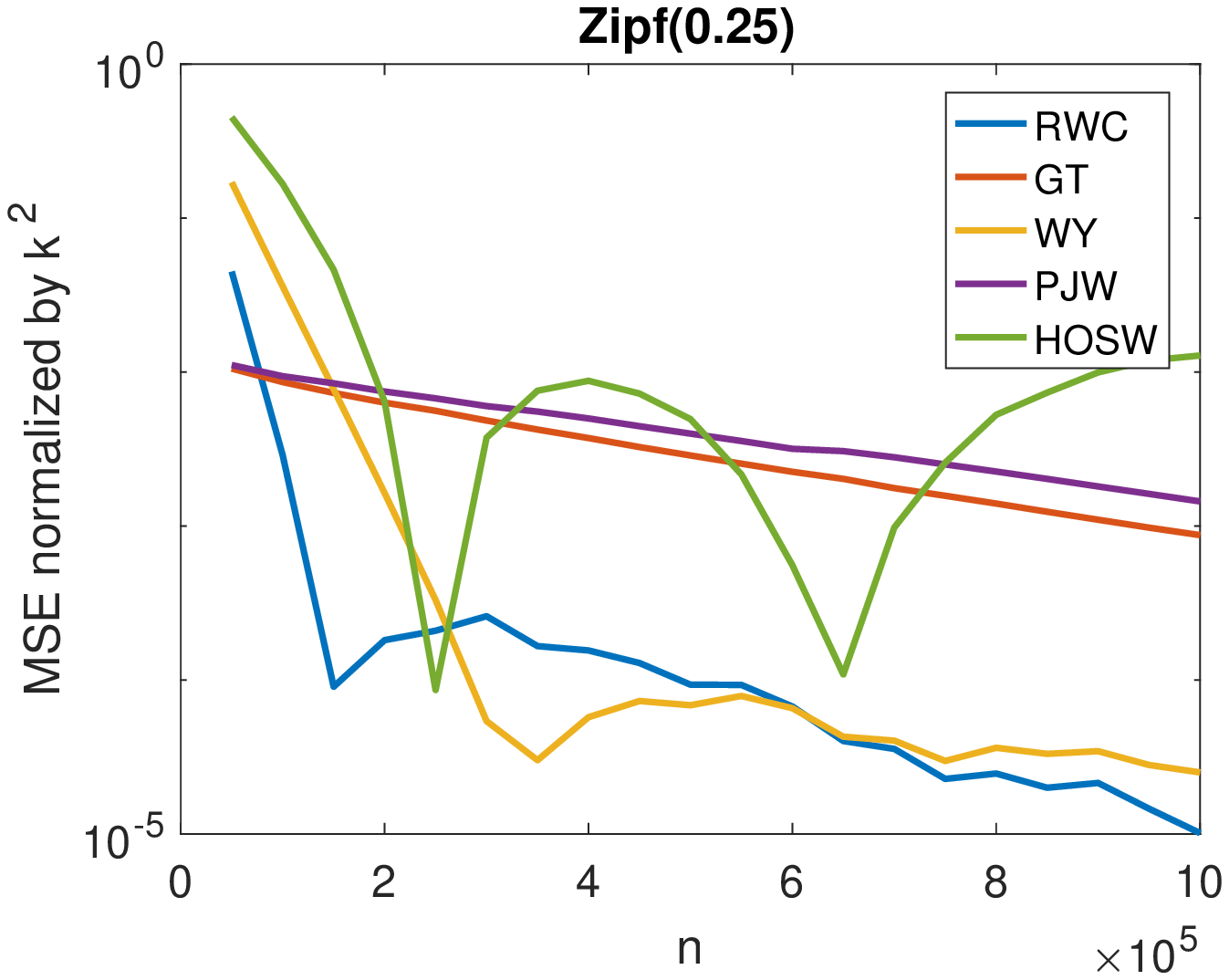}}
  \subfigure[Uniform distribution.]{\includegraphics[width=0.245\linewidth]{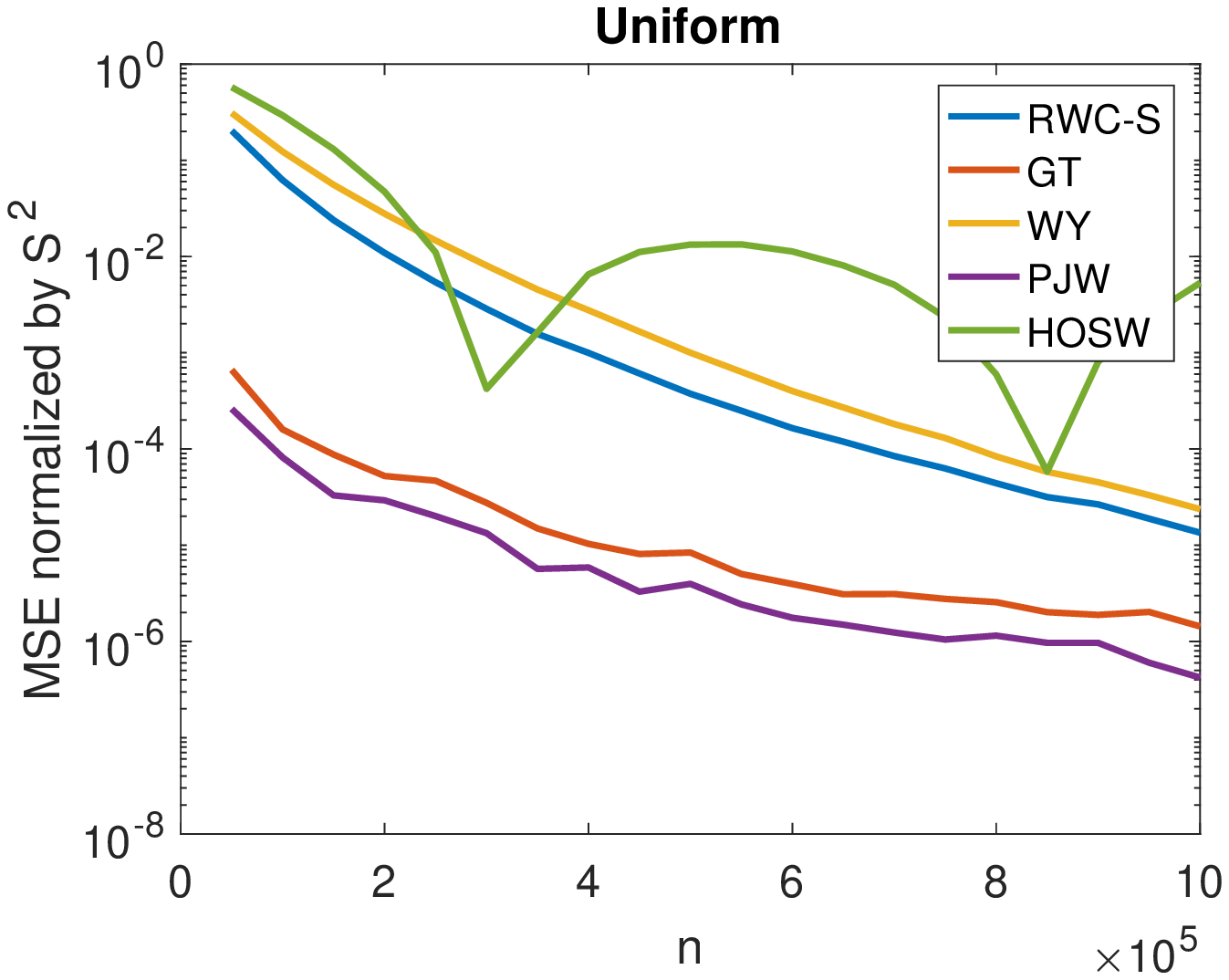}}
  \subfigure[Benford distribution.]{\includegraphics[width=0.245\linewidth]{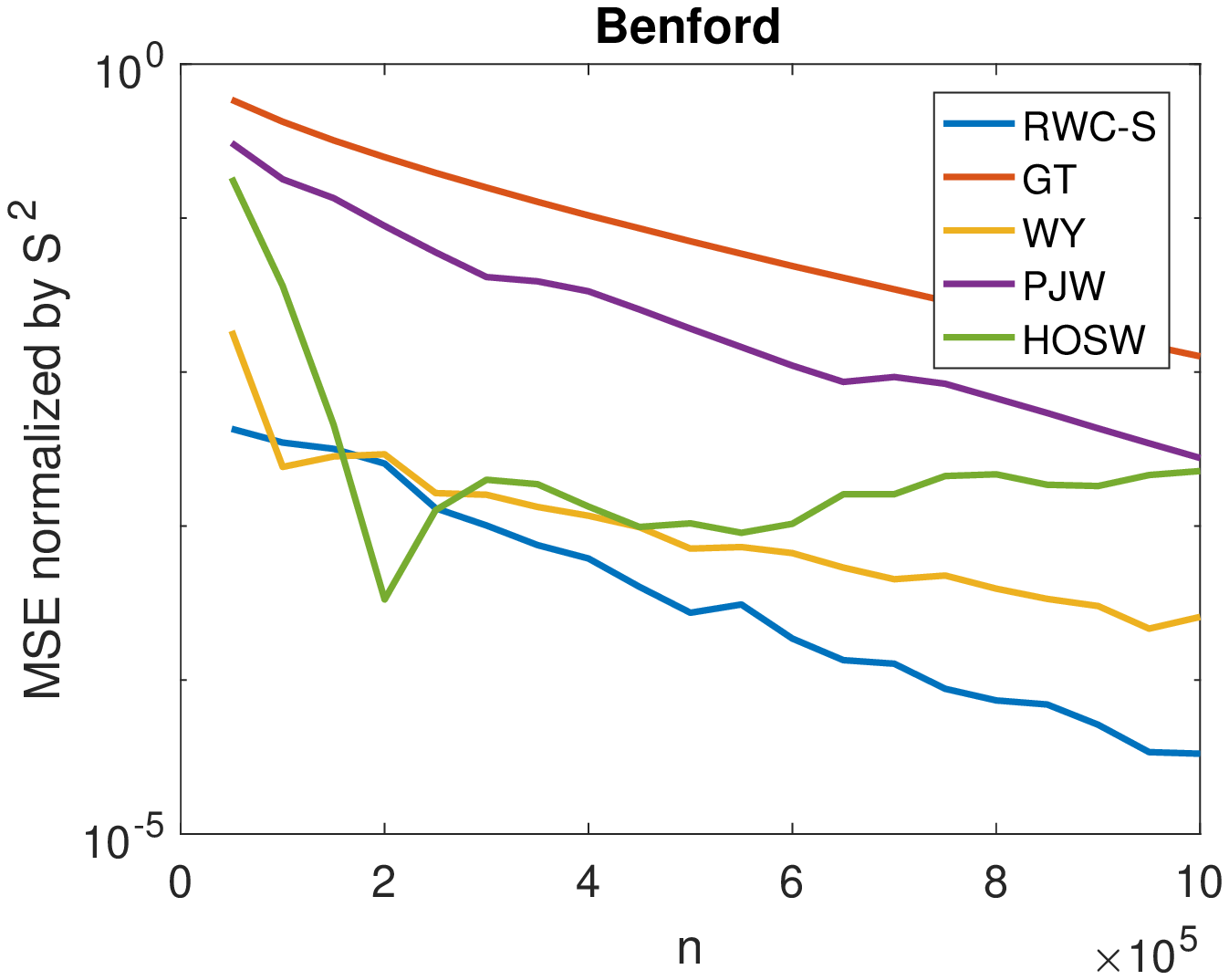}}
  \subfigure[Zipf$(1.5)$ distribution.]{\includegraphics[width=0.245\linewidth]{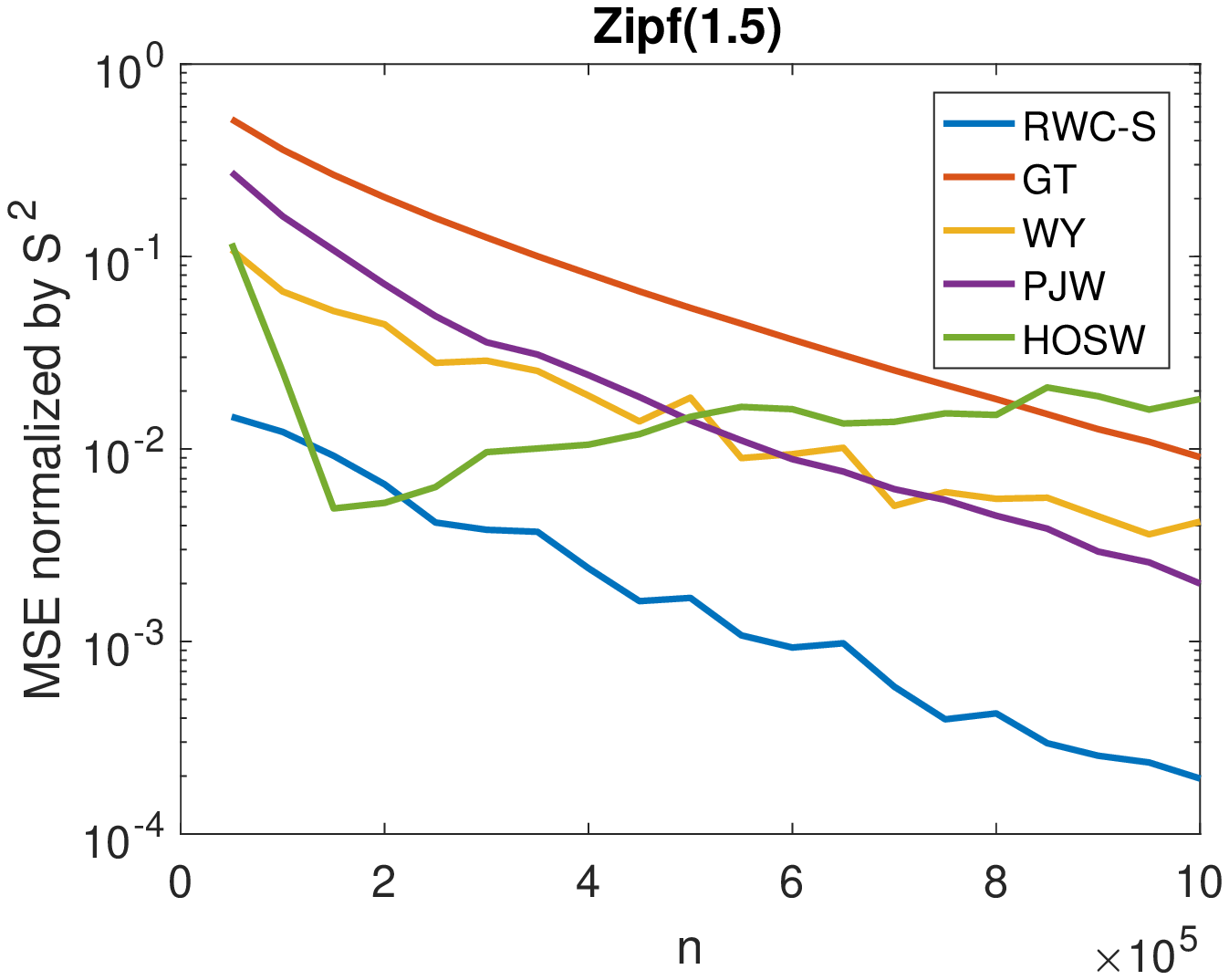}}
  \subfigure[Zipf$(0.25)$ distribution.]{\includegraphics[width=0.245\linewidth]{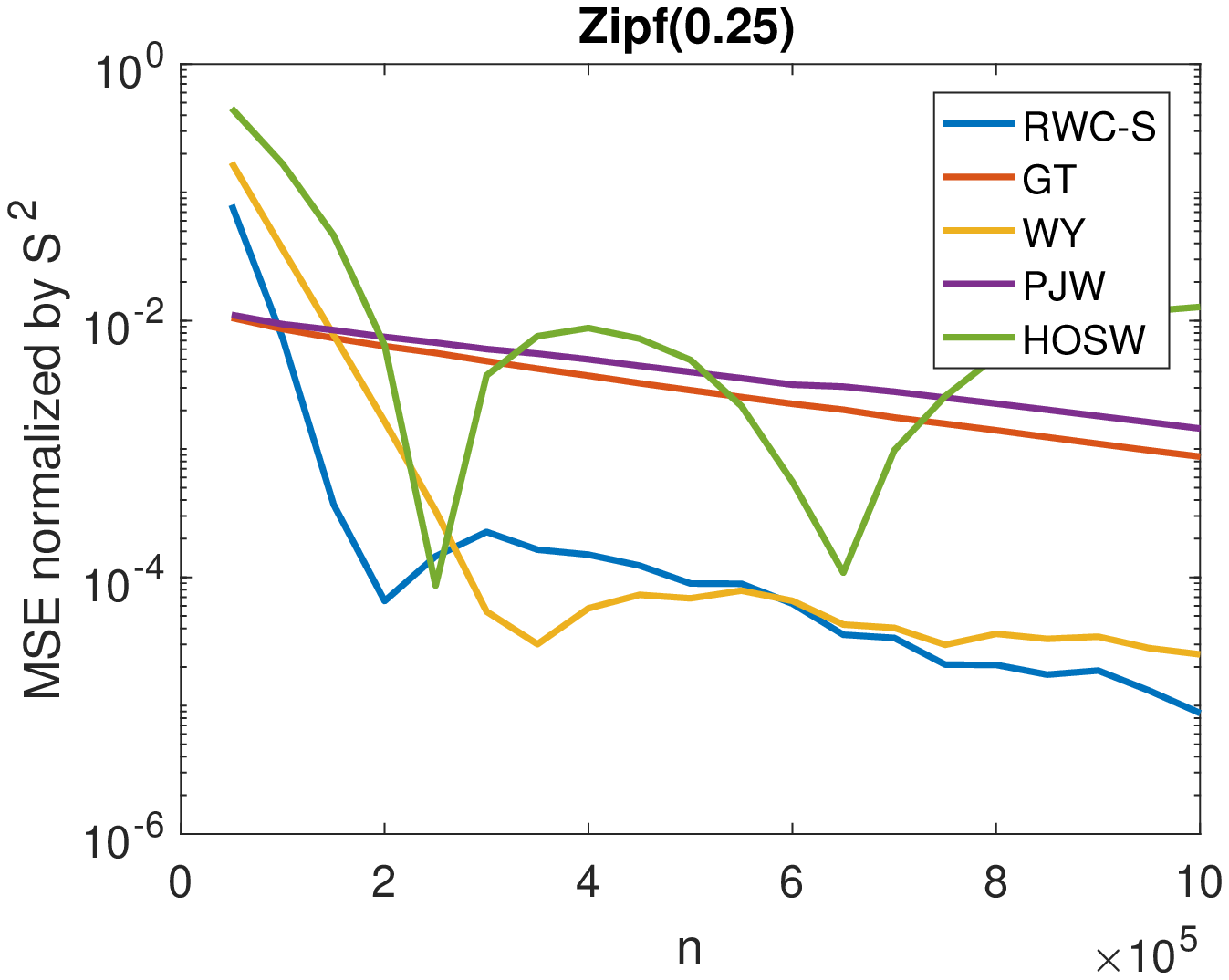}}
  \subfigure[Uniform distribution.]{\includegraphics[width=0.245\linewidth]{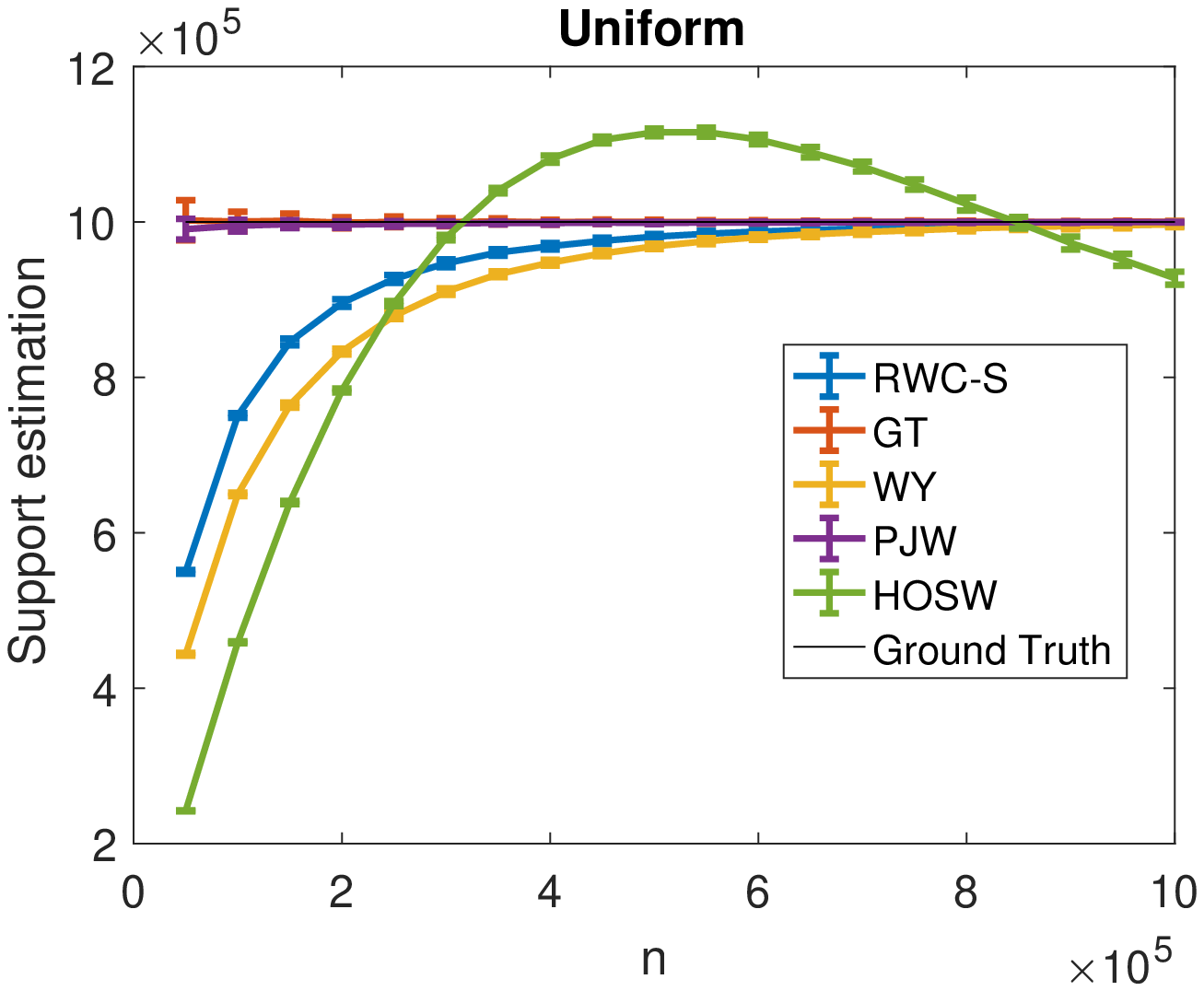}}
  \subfigure[Benford distribution.]{\includegraphics[width=0.245\linewidth]{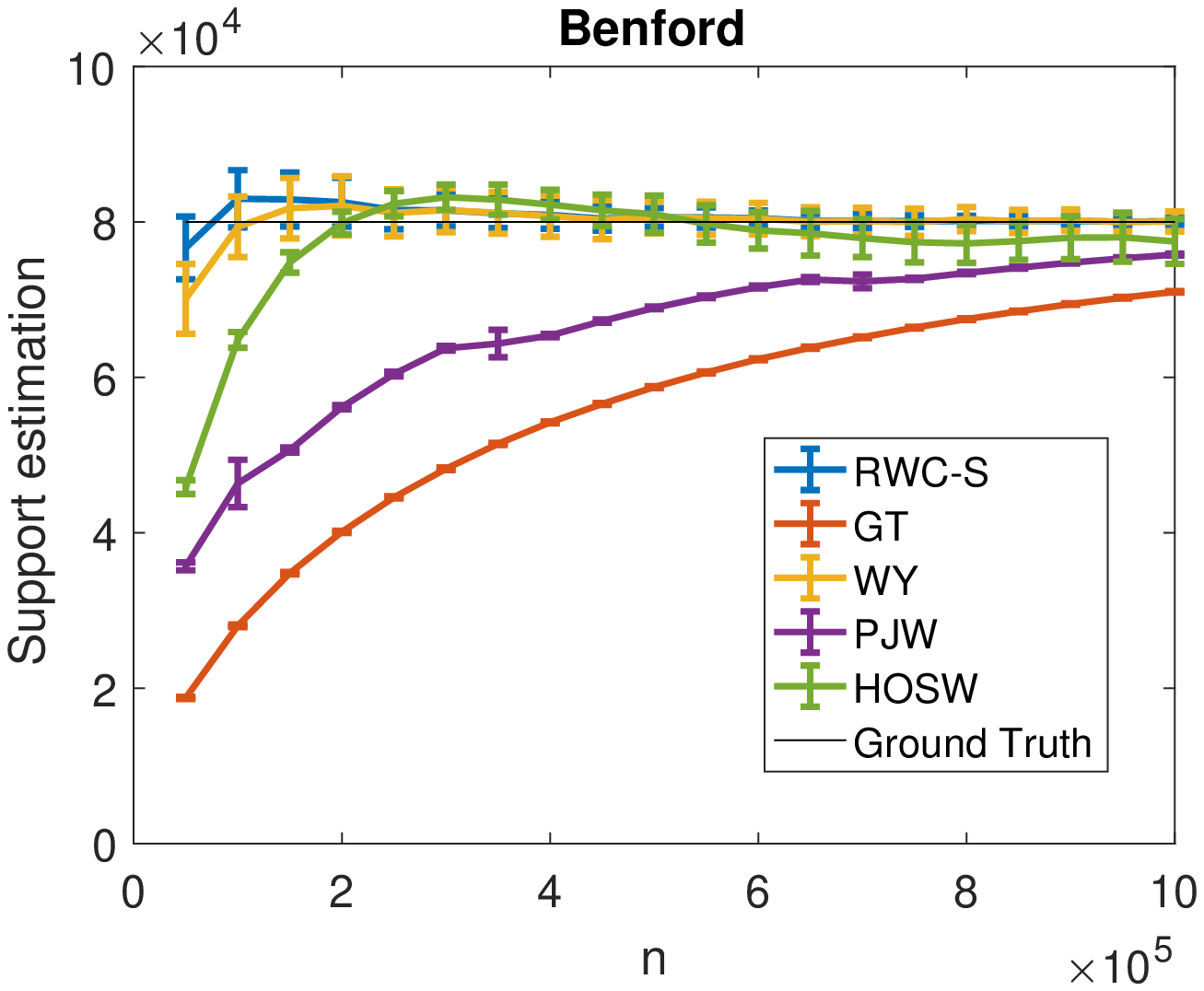}}
  \subfigure[Zipf$(1.5)$ distribution.]{\includegraphics[width=0.245\linewidth]{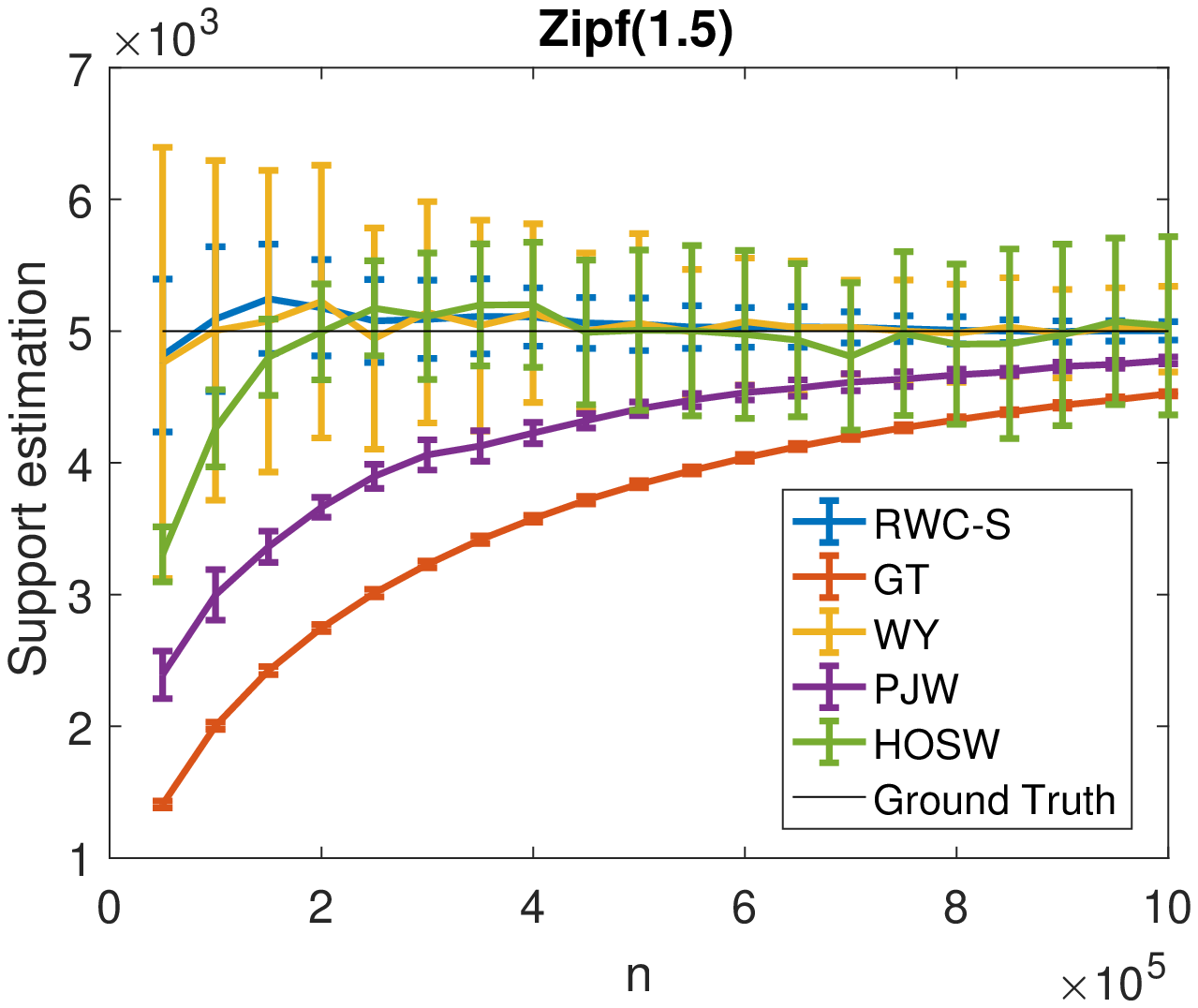}}
  \subfigure[Zipf$(0.25)$ distribution.]{\includegraphics[width=0.245\linewidth]{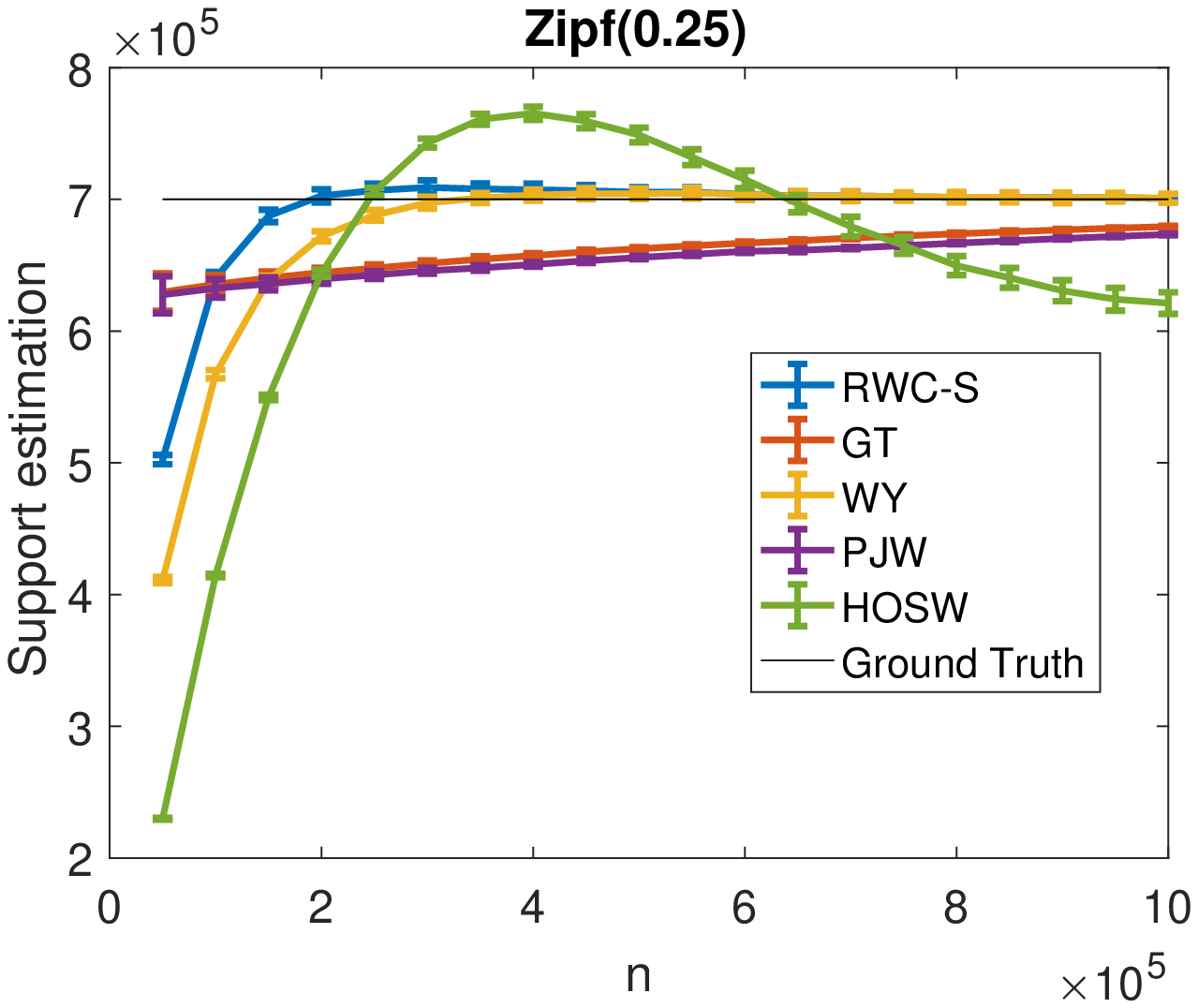}}
  \vspace{-0.2cm}
  \caption{ Figure (a) illustrates the importance of exponential weights for reducing the bias. The WY estimator aims to minimize the "worst case" risk which arises for small values of $\lambda$ (close to uniform). As $\lambda$ increases, the variance dominate the bias (i.e. Figure (o)). Figure (b) gives an example for the coefficients $g_L$ of the  RWC and WY estimators. We set $n = k = 10^6$ and $c_0 = 0.558$ for both Figure (a) and (b). The remaining figures plot the MSE of all estimators considered. The $y$-axis is on the log scale. Figures (m)-(p) show the mean and variance of the estimators for the tested distributions.}\label{fig:eachdist}
\end{figure*}
Considering the exponential weighting term when designing the estimator leads to significant performance improvements in the bias of the estimator, as illustrated in Figure~\ref{fig:BiasforexpW}. There, we see that a classical Chebyshev approximation introduces a larger bias then the weighted one whenever the underlying distribution is close to uniform (i.e., when $\lambda \sim \frac{n}{k} = 1$). This phenomenon persists even when regularizations is taken into account, as shown in~Figure~\ref{fig:eachdist}(m). Furthermore, regularized weighted Chebyshev estimators exhibit a different behavior than that observed for classical approximations: Figure~\ref{fig:gLcompare} provides an example for the values of $g_L$ used in RWC that do not alternate in sign. Nevertheless, these and all other estimators tested run for $\leq 1$s on sample sizes $\approx 10^6$.

One common approach to testing support estimators on real data is to estimate the number of distinct words in selected books. The authors of~\cite{wu2019chebyshev,valiant2013estimating} applied their estimators to Hamlet in order to estimate the number of words used in the play. Here, we repeat this experiment not only on Hamlet but Othello, Macbeth and King Lear as well. The comparative results for different estimators, following the setup in~\cite{wu2019chebyshev}, are presented in Figure~\ref{fig:Hamlet}. In the experiments, we randomly sampled words in the text with replacement and used the obtained counts to estimate the number of distinct words. For simplicity, we set $k$ to equal the total number of words. For example, as the number of total words in Hamlet equals $30,364,$ we set $k = 30,364$.
\begin{figure*}[!htb]
  \centering
  \subfigure[]{\includegraphics[width=0.245\linewidth]{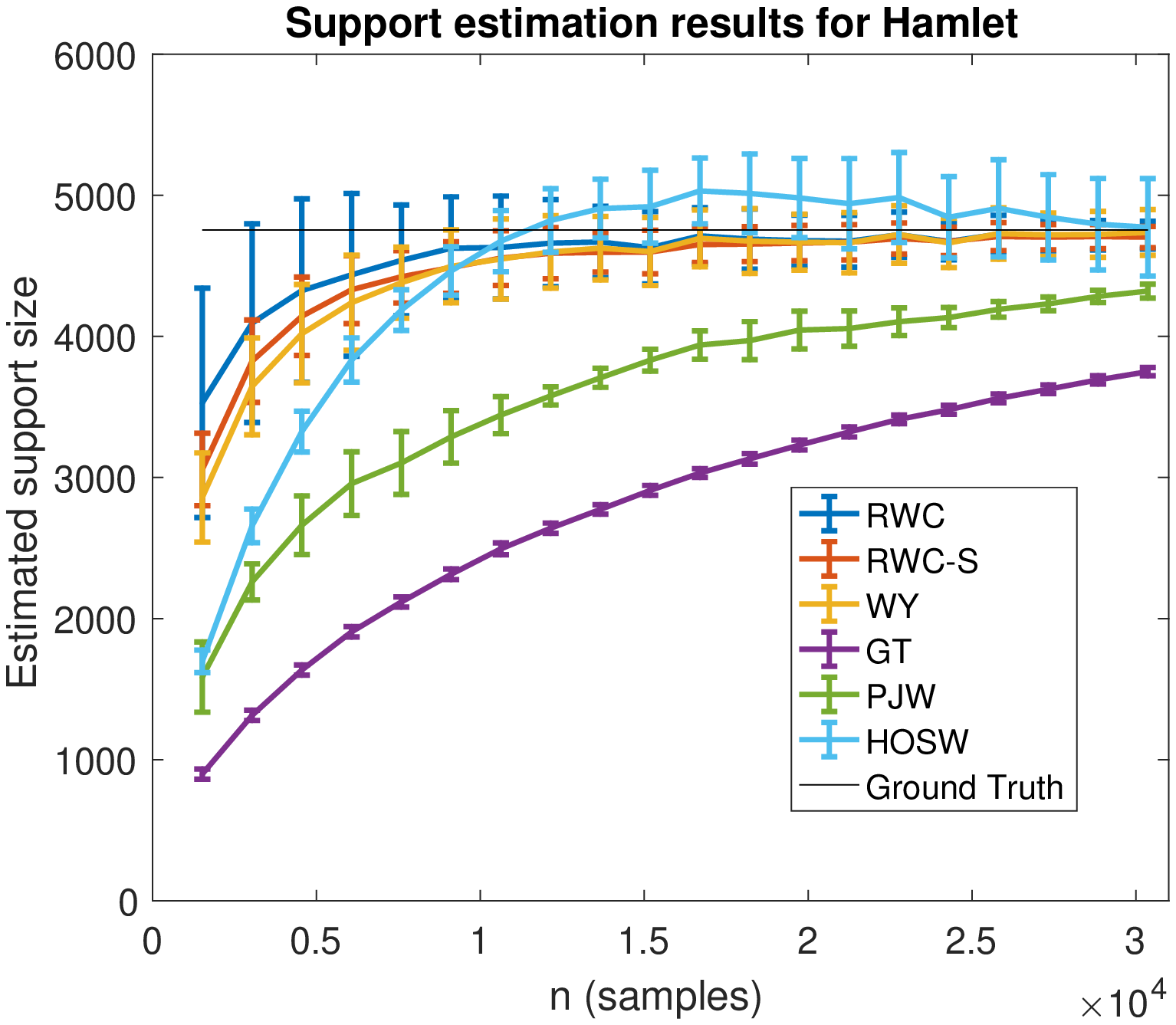}}
  \subfigure[]{\includegraphics[width=0.245\linewidth]{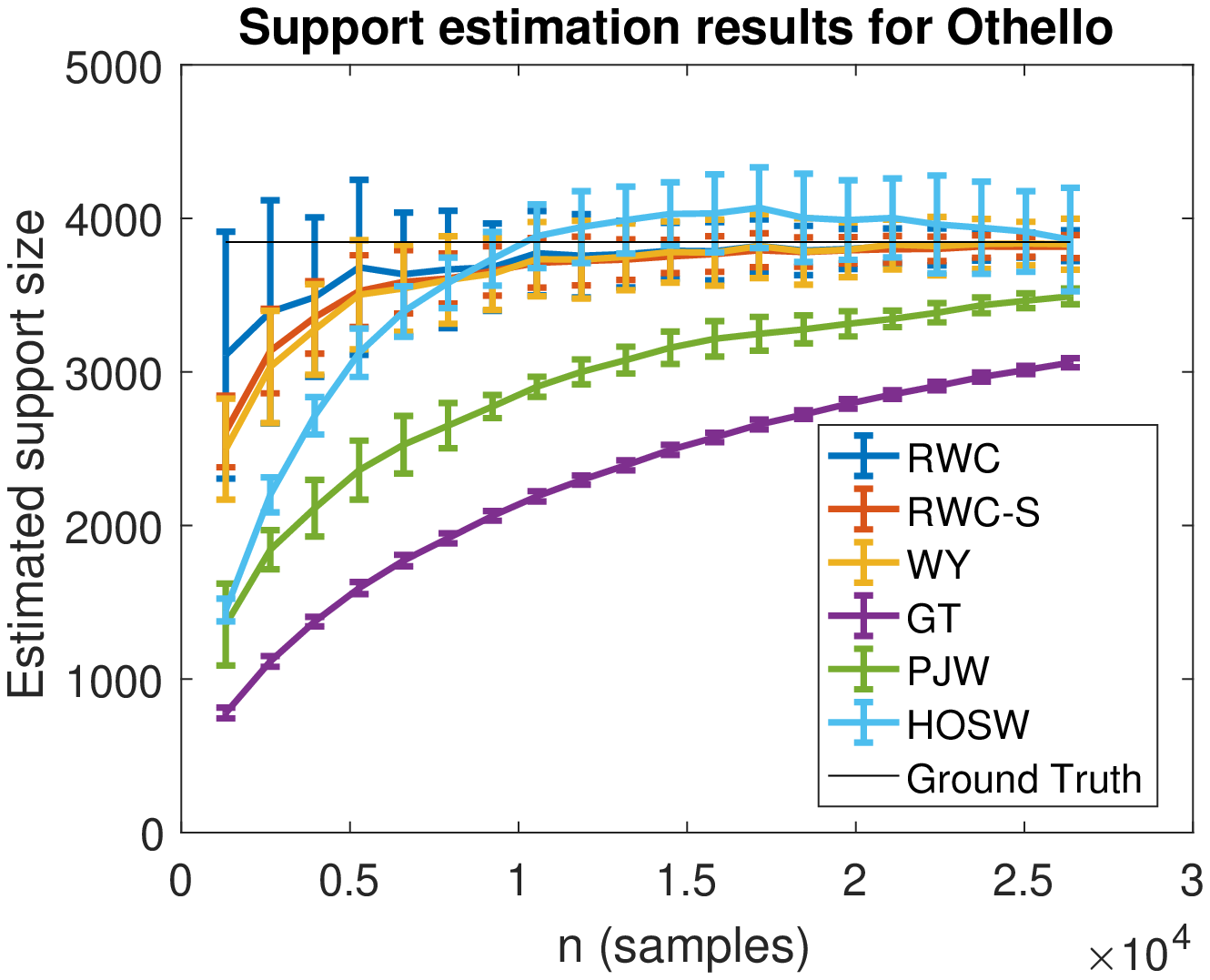}}
  \subfigure[]{\includegraphics[width=0.245\linewidth]{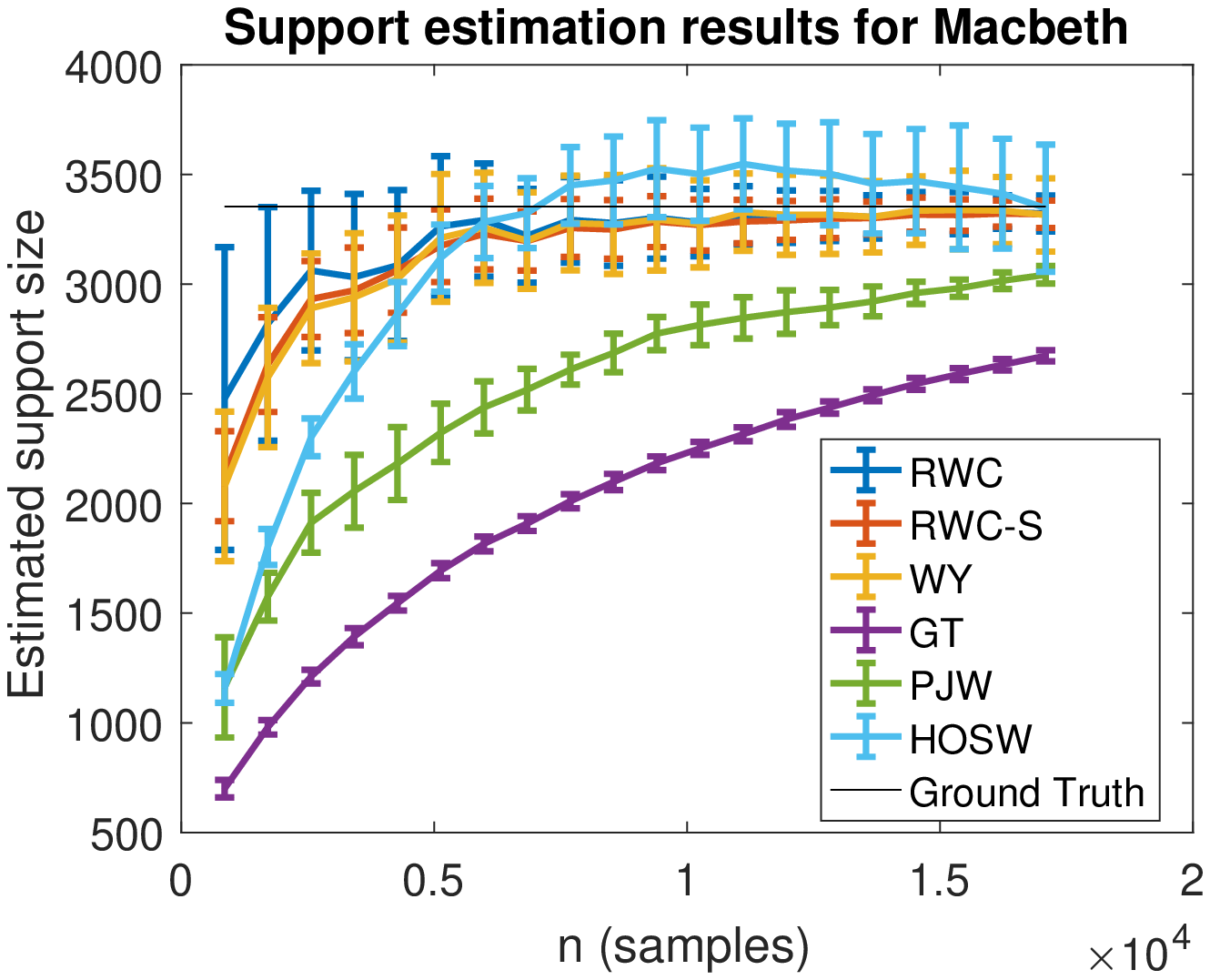}}
  \subfigure[]{\includegraphics[width=0.245\linewidth]{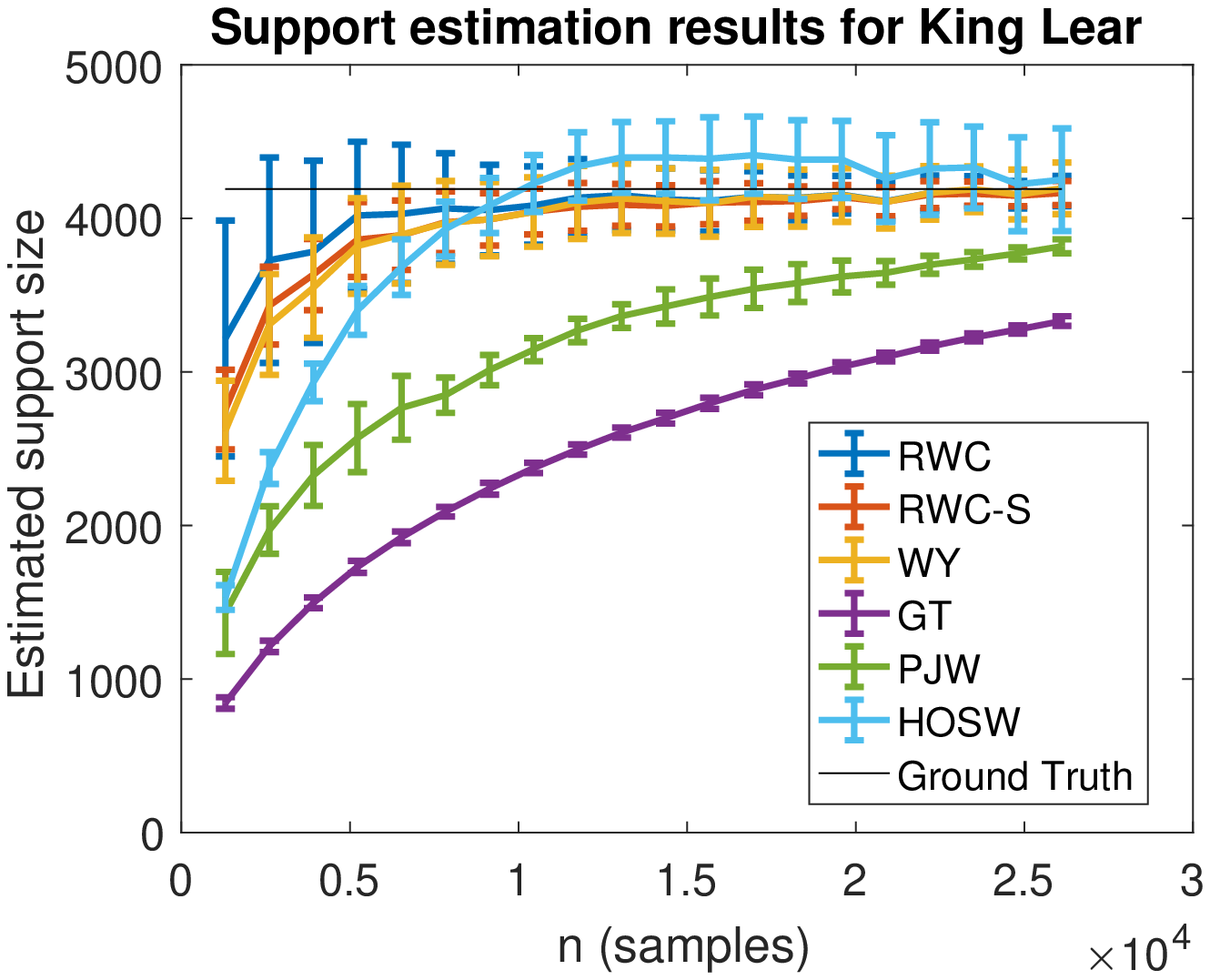}}\\
  \subfigure[]{\includegraphics[width=0.245\linewidth]{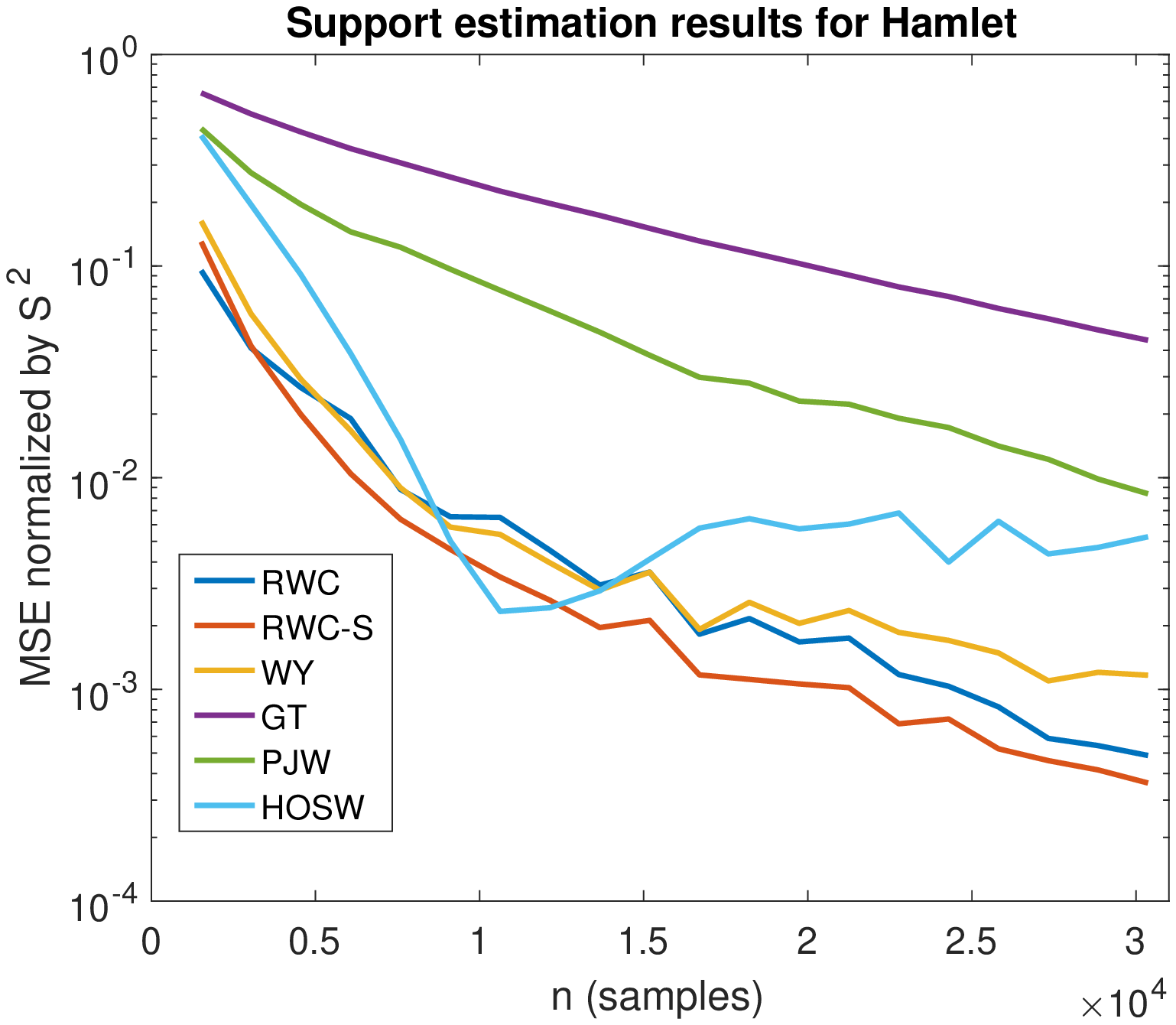}}
  \subfigure[]{\includegraphics[width=0.245\linewidth]{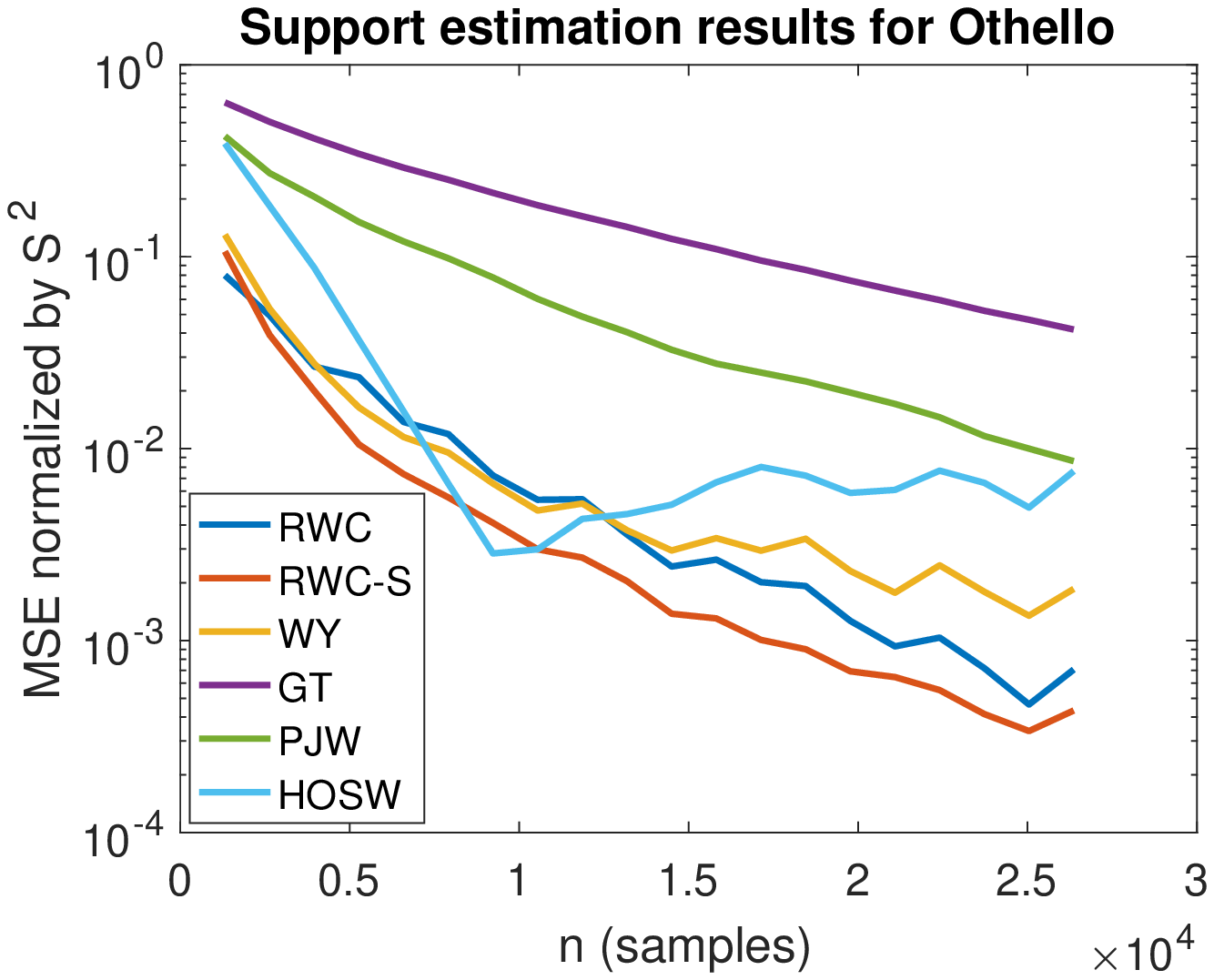}}
  \subfigure[]{\includegraphics[width=0.245\linewidth]{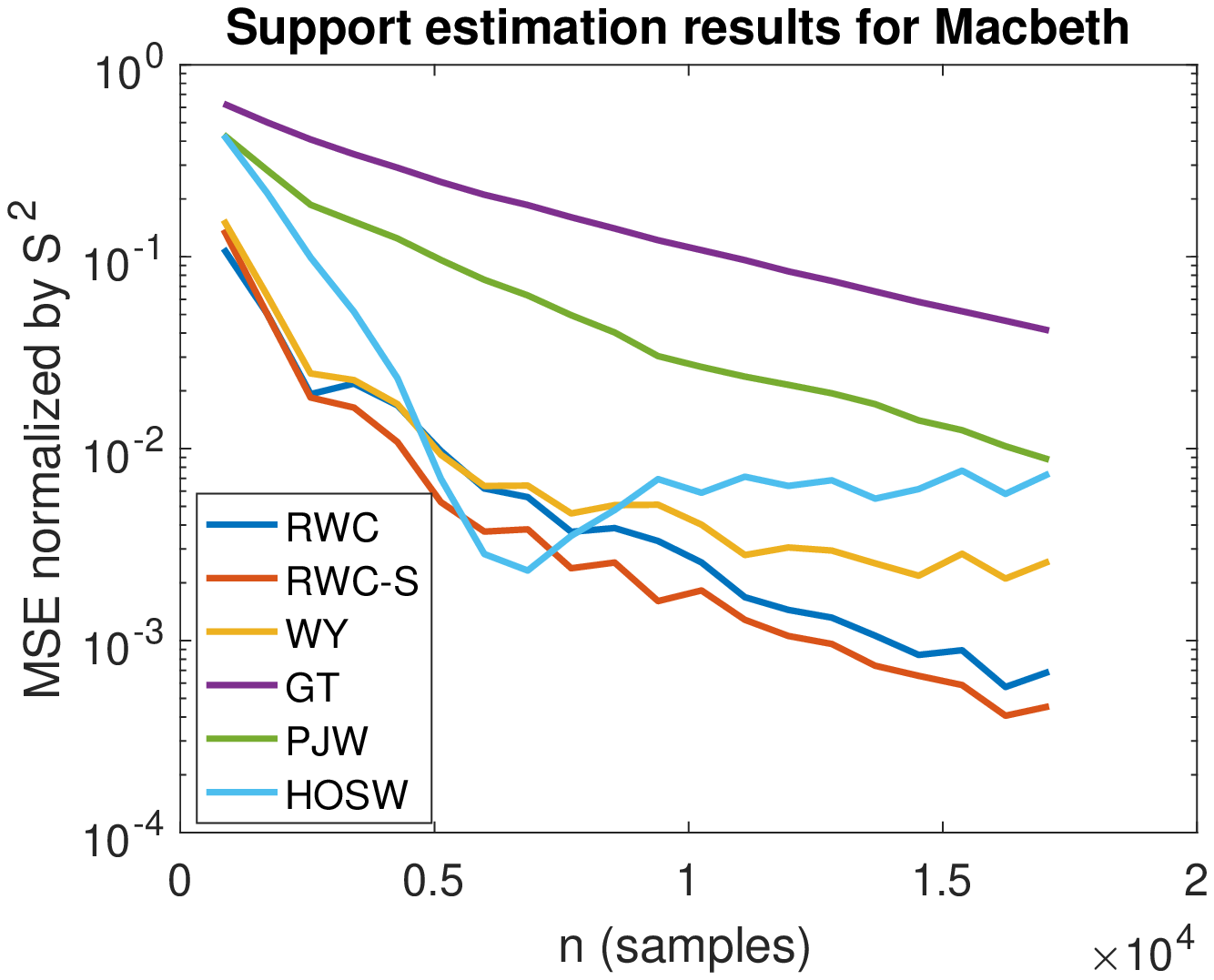}}
  \subfigure[]{\includegraphics[width=0.245\linewidth]{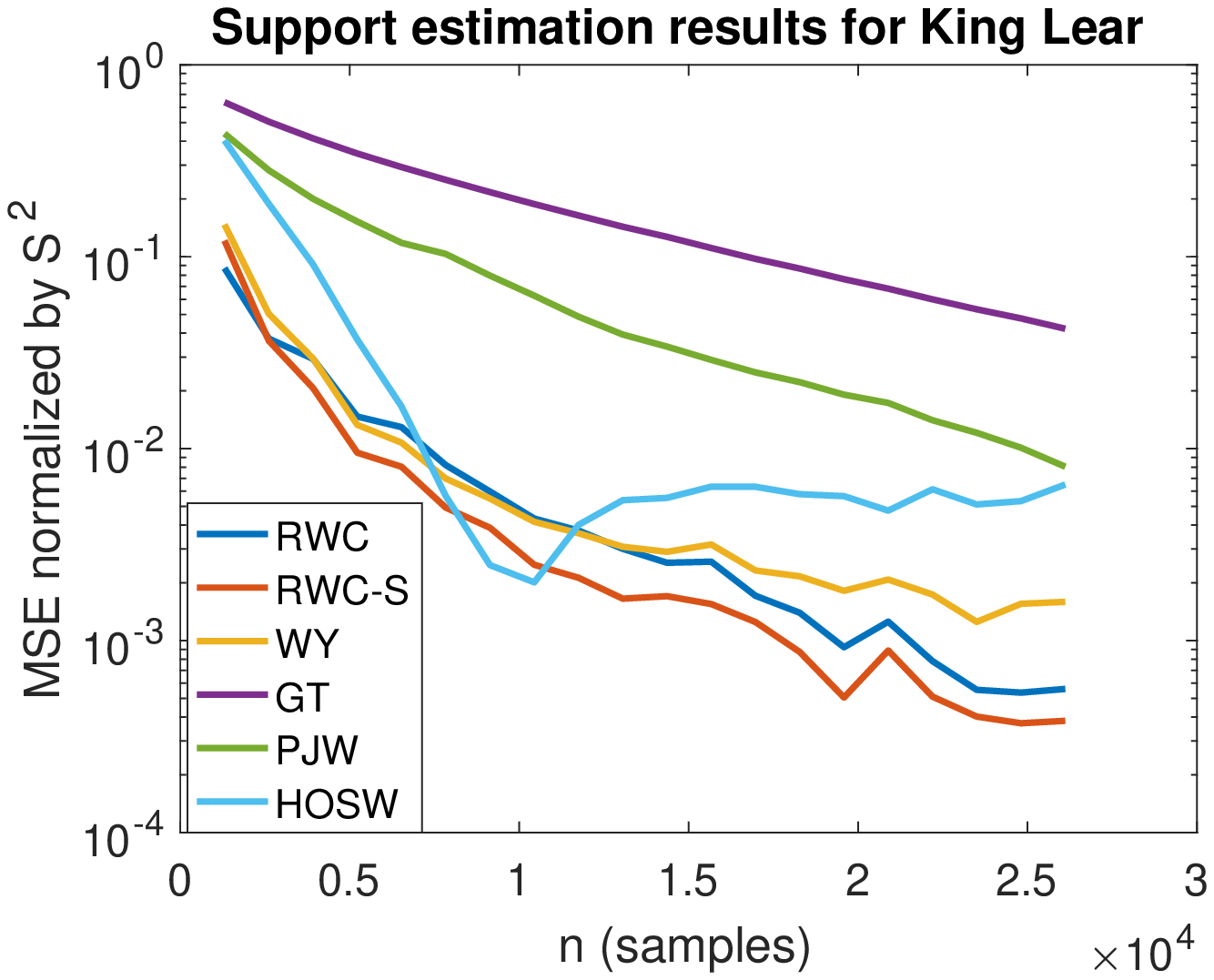}}
  \vspace{-0.2cm}
  \caption{The result is obtained over $100$ independent trials. The first row of figures shows the mean and standard deviation of the estimators, while the second row of figures shows the MSE normalized by $S^2$.}\label{fig:Hamlet}
\end{figure*}
As may be clearly seen, our methods significantly outperforms all other competitive techniques both in terms of convergence rate and the accuracy of the estimated support for all experiments. These results further strengthen the case for the practical of our estimators.

We now turn our attention to a new support estimation problem, concerned with determining the bacterial diversity of the human gut microbiome. Although it is known that human guts hosts trillions of bacterial cells~\cite{sender2016revised}, very little effort has been placed to rigorously estimate the actual number of different bacterial genera in the gut. To address this problem, we retrieved $1374$ human gut microbiome datasets from the NIH Human Gut Microbiome~\cite{peterson2009nih} and the America Gut Microbiome websites (http://americangut.org). To determine which bacterial species are presented in the samples, we ran Kraken~\cite{wood2014kraken}, a taxonomy profiler, on each dataset, using a library of size $8$ GB. This lead to $n = 7,415,847$ samples, which we used for obtaining the bacterial genera sample histograms (depicted in Figure~\ref{fig:guts}), and to specify $k=n$. The obtained estimates are listed in Table~\ref{tab:guts} below.
\begin{figure}[!htb]
\centering
  \subfigure[Entire histogram]{\includegraphics[width=0.32\linewidth]{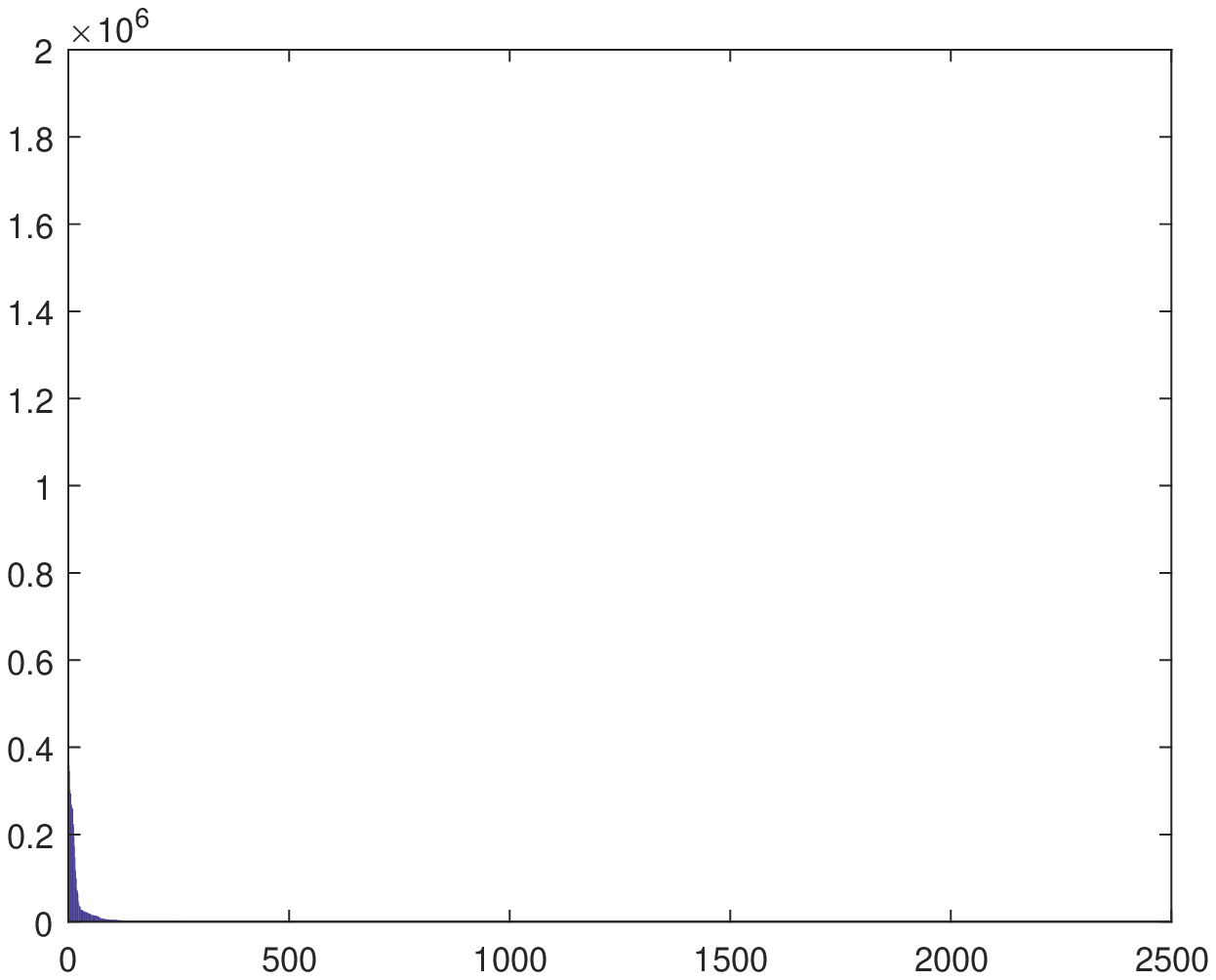}}
  \subfigure[Top $100$ ]{\includegraphics[width=0.32\linewidth]{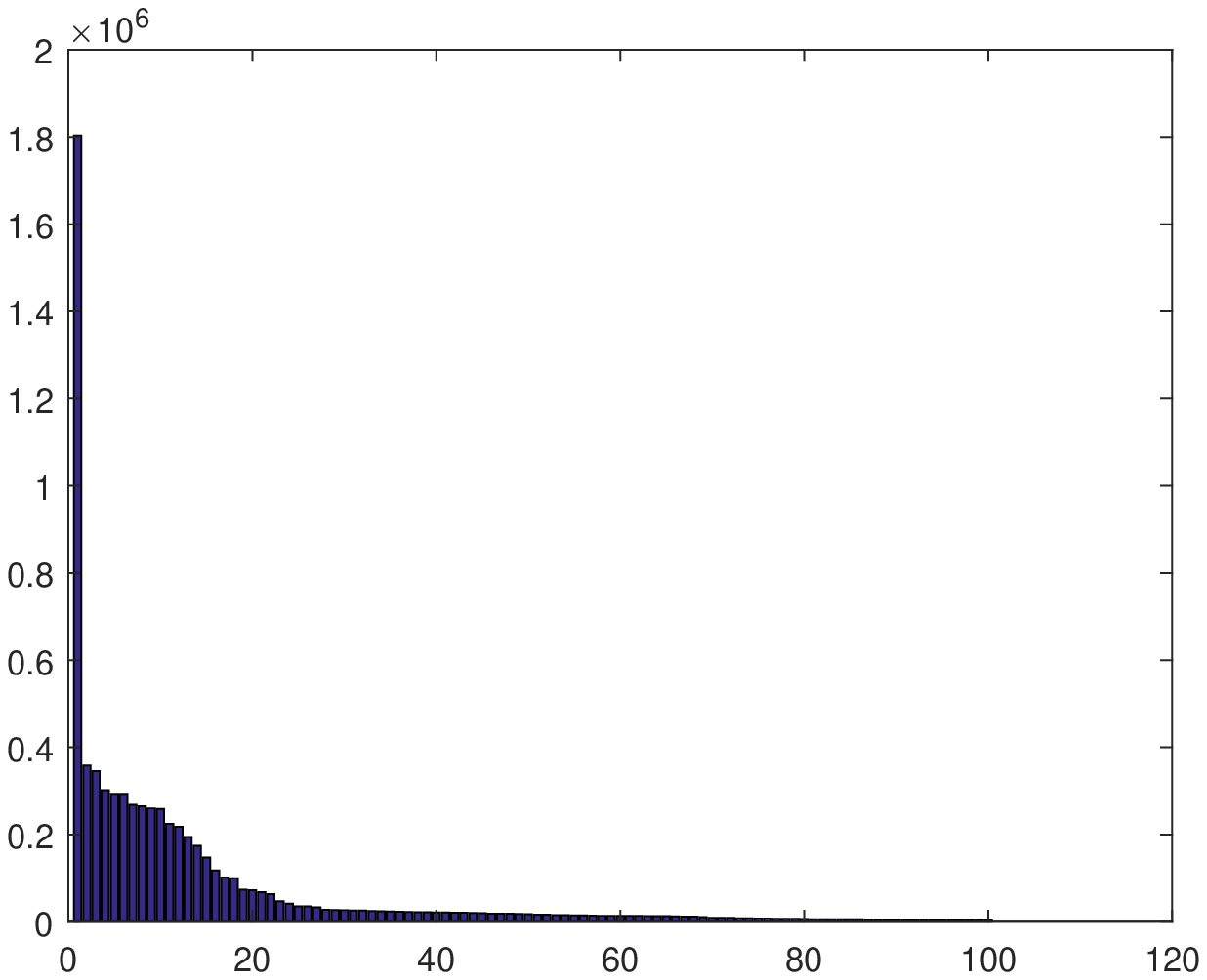}}
  \subfigure[Top $25$ ]{\includegraphics[width=0.32\linewidth]{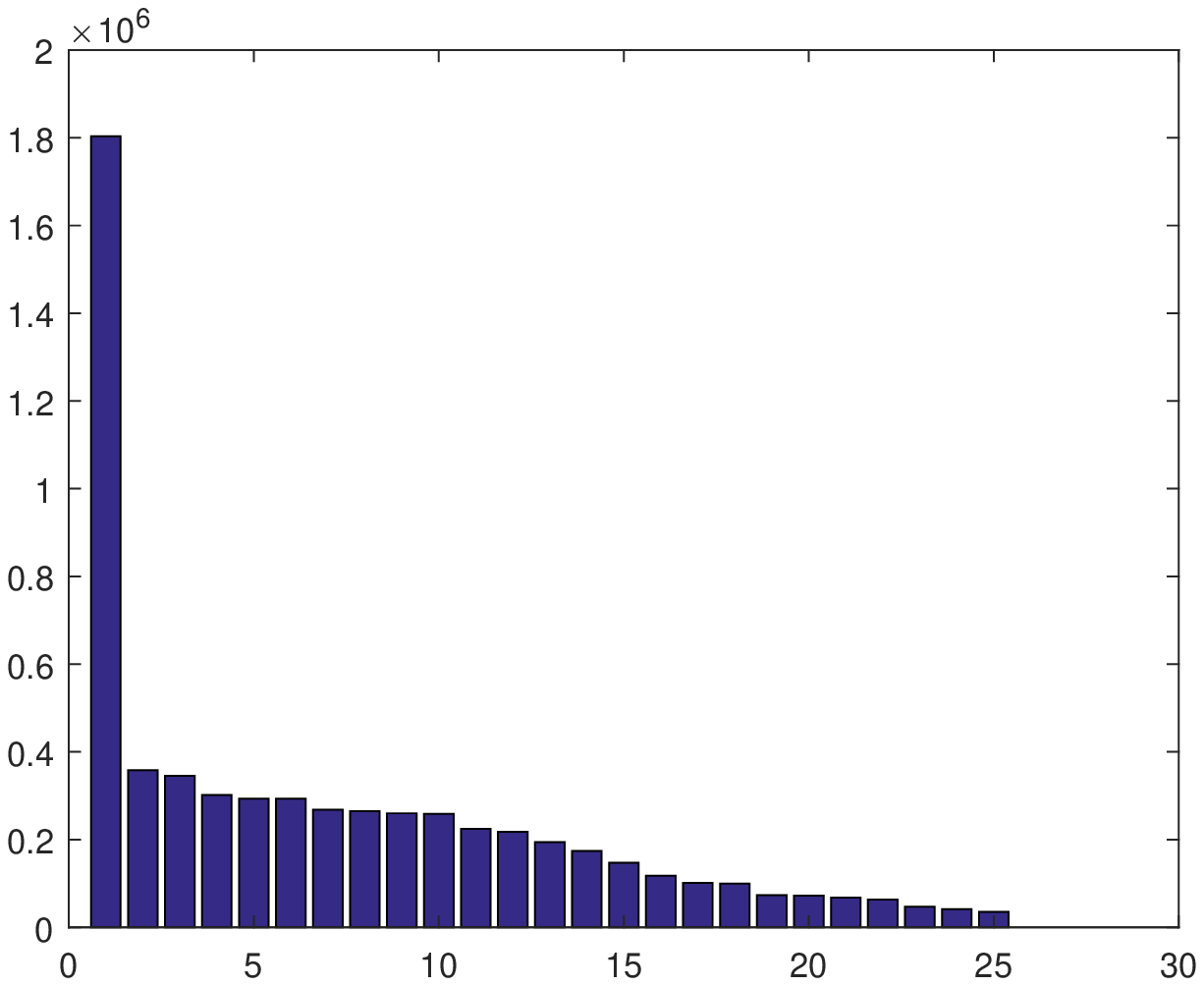}}
  \caption{The histogram of bacterial genera in the human gut. The plots (a)-(c) correspond to the entire histogram, the histogram for the $100$ most frequently encountered genera, and the histogram for the $25$ most frequently encountered genera, respectively.}\label{fig:guts}
\end{figure}
\tiny
\begin{table}[ht]
\centering
\begin{tabular}[t]{ccccccc}
\toprule
 RWC & RWC-S & PJW & Naive & GT & WY & HOSW\\
\midrule
  2364 & 2346 & 2250 & 2151 & 2151 & 2663 & 2673\\
\bottomrule
\end{tabular}
\vspace{0.1in}
\caption{Estimates for the number of different bacterial genera in the human gut obtained using four different methods. The numbers in brackets correspond to the parameter $c_0$.}\label{tab:guts}
\end{table}%
\normalsize
\bibliography{example_paper}
\bibliographystyle{IEEEtran}
%
\appendices
\section{Proof of Theorem~\ref{lma:interval_sq}}

     To prove the result, we need to show that $\forall \lambda\geq 6.5L$, $\frac{\partial }{\partial \lambda}g(\mathbf{a},\lambda)<0$. The derivative of the first term in $g$ equals
    \begin{align*}
        & \frac{\partial}{\partial \lambda}\frac{1}{k}\bigg(\sum_{l=0}^{L}e^{-\lambda}a_l^2\lambda^ll!\bigg) = \frac{1}{k}\bigg(\sum_{l=0}^{L}(\frac{l}{\lambda}-1)e^{-\lambda}a_l^2\lambda^ll!\bigg).
    \end{align*}
    Clearly, the right hand side in the above expression is negative for all $\lambda>L$. The second term of the derivative equals
    \begin{align*}
        & \frac{\partial}{\partial \lambda}\bigg(e^{-\lambda}P_L(\lambda,\mathbf{a})\bigg)^2 \\
        & = 2\bigg(e^{-\lambda}P_L(\lambda,\mathbf{a})\bigg)\bigg(-e^{-\lambda}P_L(\lambda,\mathbf{a}) + e^{-\lambda}\sum_{l=0}^{L}
        \frac{l}{\lambda} a_l\lambda^l\bigg)\\
        & = 2e^{-2\lambda}\bigg(P_L(\lambda,\mathbf{a})\bigg)\bigg(\sum_{l=0}^{L}
        (\frac{l}{\lambda}-1) a_l\lambda^l\bigg).
    \end{align*}
    To analyze the two terms of the derivative, we introduce the vectors $\mathbf{y},\mathbf{z},\mathbf{1}$ and the diagonal matrix $\mathbf{D}$ according to
    \begin{align*}
        &\mathbf{y} = (a_0\lambda^0,a_1\lambda^1,...,a_L\lambda^L)^T,\\
        &\mathbf{z} = ((\frac{0}{\lambda}-1),(\frac{1}{\lambda}-1),...,(\frac{L}{\lambda}-1))^T,\\
        &\mathbf{1} = (1,1,...,1)^T,\\
        & D_{ii} = (-1+\frac{i-1}{\lambda})\frac{(i-1)!}{\lambda^{(i-1)}}.
    \end{align*}
    Consequently, we have
    \begin{align*}
        & \frac{\partial}{\partial \lambda}\frac{1}{k}\bigg(\sum_{l=0}^{L}e^{-\lambda}a_l^2\lambda^ll!\bigg) = \frac{e^{-\lambda}}{k}\mathbf{y}^T\mathbf{D}\mathbf{y},\\
        & \frac{\partial}{\partial \lambda}\bigg(e^{-\lambda}P_L(\lambda,\mathbf{a})\bigg)^2 = 2e^{-2\lambda}\mathbf{y}^T\mathbf{1}\mathbf{z}^T\mathbf{y} \\
        & = e^{-2\lambda}\mathbf{y}^T(\mathbf{1}\mathbf{z}^T+\mathbf{z}\mathbf{1}^T)\mathbf{y}.
    \end{align*}
    Therefore,
    \begin{align*}
        \frac{\partial}{\partial \lambda}g(\mathbf{a},\lambda) = e^{-2\lambda}\mathbf{y}^T\left ( \frac{e^{\lambda}}{k}\mathbf{D} +(\mathbf{1}\mathbf{z}^T+\mathbf{z}\mathbf{1}^T)  \right ).\mathbf{y}
    \end{align*}
    To show that $\frac{\partial}{\partial \lambda}g(\mathbf{a},\lambda)<0$ for all polynomials of degree $L$ whenever $\lambda > CL$, we show that the matrix $\left ( \frac{e^{\lambda}}{k}\mathbf{D} +(\mathbf{1}\mathbf{z}^T+\mathbf{z}\mathbf{1}^T)  \right )$ is negative-definite whenever $\lambda > CL$, for some constant $C>0$. It suffices to show that the sum of the maximum eigenvalues of $\frac{e^{\lambda}}{k}\mathbf{D}$ and $(\mathbf{1}\mathbf{z}^T+\mathbf{z}\mathbf{1}^T)$ is negative, since $\frac{e^{\lambda}}{k}\mathbf{D}$ is a diagonal matrix. Thus, we turn our attention to determining the maximum eigenvalues of these two matrices. For $\frac{e^{\lambda}}{k}\mathbf{D}$, the maximum eigenvalue satisfies
    \begin{align*}
        \frac{e^{\lambda}}{k}\max_{i\in\{0,1,...,L\}} \left(-1+\frac{i}{\lambda}\right) \frac{i!}{\lambda^{i}} \leq -\frac{e^{\lambda}}{2k}\min_{i\in\{0,1,...,L\}}\frac{i!}{\lambda^{i}},
    \end{align*}
    since for $\lambda>2L$, one has $(-1+\frac{i}{\lambda})\leq -\frac{1}{2}$. When $\lambda>L$, it is clear that $\frac{i!}{\lambda^{i}}$ is decreasing in $i$, for $i\in\{0,1,...,L\}$, so that
    \begin{align*}
        \min_{i\in\{0,1,...,L\}}\frac{i!}{\lambda^{i}} = \frac{L!}{\lambda^L} \geq \left(\frac{L}{e\lambda}\right)^L.
    \end{align*}
    The last inequality is a consequence of Stirling's formula, which asserts that $n! \geq (\frac{n}{e})^n$. Combining the above expressions, we obtain
    \begin{align*}
        \frac{e^{\lambda}}{k}\max_{i\in\{0,1,...,L\}}\left(-1+\frac{i}{\lambda}\right)\frac{i!}{\lambda^{i}} \leq -\frac{e^{\lambda}}{2k}\left(\frac{L}{e\lambda}\right)^L.
    \end{align*}
    Next, we derive an upper bound on maximum eigenvalue of the second matrix. The $i,j$ entry of the matrix $(\mathbf{1}\mathbf{z}^T+\mathbf{z}\mathbf{1}^T)$ equals $\frac{i+j-2}{\lambda} - 2,$ and all these values are negative when $\lambda>L$. Moreover, it is clear that the matrix of interest has rank equal to $2$.
Therefore, the matrix has exactly two non-zero eigenvalues.

Let $\mathbf{A} = -(\mathbf{1}\mathbf{z}^T+\mathbf{z}\mathbf{1}^T)$. All entries of $\mathbf{A}$ are positive whenever $\lambda>L$. By Gershgorin's theorem, we can upper bound the maximum eigenvalues of the matrix $\mathbf{A}$ by its maximum row sum. It is obvious that the maximum row sum equals
    \begin{align*}
        2(L+1) - \frac{L(L+1)}{2\lambda}.
    \end{align*}
    Moreover, the trace of $\mathbf{A}$ equals
    \begin{align*}
        2(L+1) - \frac{L(L+1)}{\lambda}.
    \end{align*}
    This implies that the minimum eigenvalue of $\mathbf{A}$ is lower bounded by $- \frac{L(L+1)}{2\lambda},$ which directly implies that the maximum eigenvalue of $(\mathbf{1}\mathbf{z}^T+\mathbf{z}\mathbf{1}^T)$ is upper bounded by $\frac{L(L+1)}{2\lambda}$.

Summing up the two previously derived upper bounds gives
    \begin{align*}
        h(\lambda)\triangleq -\frac{e^{\lambda}}{2k}\left(\frac{L}{e\lambda}\right)^L+\frac{L(L+1)}{2\lambda},
    \end{align*}
    whenever $\lambda>2L$.
    Note that $h(\lambda)<0$ is equivalent to
    \begin{align}
        &\frac{L(L+1)}{2\lambda} < \frac{e^{\lambda}}{2k}\left(\frac{L}{e\lambda}\right)^L\nonumber\\
        & \Leftrightarrow \log(L) + \log(L+1) + \log(k) - L\log(L) + L \nonumber\\
        &< \lambda + \log(\lambda) - L\log(\lambda).\label{app:eq1}
    \end{align}
    The function $\lambda + \log(\lambda) - L\log(\lambda)$ is non-decreasing in $\lambda$ whenever $\lambda> L$ since
    \begin{align*}
        \frac{d}{d\lambda}(\lambda + \log(\lambda) - L\log(\lambda)) = 1-\frac{L-1}{\lambda}.
    \end{align*}
    By the definition of $L = \lfloor c_0\log(k) \rfloor$, we also have $\log(k)\leq \frac{L+1}{c_0}$. Using $\log(x+1)\leq x,$ which holds $\forall x\geq 1$. Hence $\forall \lambda > CL$ where $C>2$, the sufficient condition for \eqref{app:eq1} to hold is
    \begin{align*}
        &\log(L) + L + \frac{L+1}{c_0} - L\log(L) + L \\
        &< CL + \log(CL) - L\log(CL).
    \end{align*}
    Rearranging terms leads to
        \begin{align*}
        \left(C-\log(C) -2 -\frac{1}{c_0} \right)L + \log(C) > \frac{1}{c_0}.
    \end{align*}
    Sufficient conditions for the above inequality to hold are $\log(C)\geq \frac{1}{c_0}$ and $(C-\log(C) -2 -\frac{1}{c_0})>0$. The first condition implies $C\geq e^{\frac{1}{c_0}} = 6.0021,$ while the second condition holds with $C = 6.5$, for which the first condition is also satisfied. This completes the proof.

\section{Proof of Theorem~\ref{thm:discretization}}
The proof consists of two parts. In the first part, we establish the conditions for convergence, while in the second part, we determine the convergence rate.
\subsection{Proof of convergence}
We start by introducing the relevant terminology. Let $\Pi \subset \mathbb{R}^{L+1}$ be a closed set of parameters, and let $f$ be a continuous functional on $\Pi$. Assume that $B \subset \mathbb{R}$ is compact and that $g:\, \Pi\mapsto \mathcal{C}(B)$ is a continuous mapping from $\Pi$ into $\mathcal{C}(B),$ where $\mathcal{C}(B)$ is the space of continuous functions over $B$ equipped with the supremum norm $||\cdot||_{\infty}$. For each $D\subset B$ let
    \begin{align*}
        M(D) = \{\mathbf{c}\in \Pi|\,g(\mathbf{c},x)\leq 0, x\in D\}
    \end{align*}
    denote the set of feasible points of the optimization problem
    \begin{align*}
        \text{min }f(\mathbf{c})\text{ over }\mathbf{c}\in M(D).
    \end{align*}
    Assuming that $M(D)\neq \emptyset$, let
    \begin{align*}
        \mu(D) = \inf\{f(\mathbf{c})|\mathbf{c}\in M(D)\},
    \end{align*}
    and define the level set
    \begin{align*}
        \text{Level}(\mathbf{c}_0,D) = \{\mathbf{c}\in \Pi |\, f(\mathbf{c})\leq f(\mathbf{c}_0)\}\cap M(D).
    \end{align*}
We also make the following two assumptions:
\begin{assumption}[Fine grid]\label{assp:grid}
    Let $\mathbb{N}_0=\mathbb{N}\cup\{0\}$. There exists a sequence $\{B_i\}$ of compact subsets of $B$ with $B_{i}\subset B_{i+1},\, i\in\mathbb{N}_0,$ for which $\lim_{i\rightarrow \infty}h(B_{i},B) = 0,$ such that
    \begin{align*}
        h(B_i,B) = \sup_{x\in B}\inf_{y\in B_i}||x-y||.
    \end{align*}
\end{assumption}
\begin{assumption}[Bounded level set]\label{assp:blevelset}
    $M(B)$ is nonempty, and there exists a $\mathbf{c}_0\in M(B)$ such that the level set $\text{Level}(\bold{c}_0,B_0)$ is bounded
    and hence compact in $\mathbb{R}^{L+1}$.
\end{assumption}
\begin{theorem}[Convergence of the discretization method, Theorem 2.1 in~\cite{reemtsen1991discretization}]\label{thm:convergence}
    Under assumptions~\ref{assp:grid} and~\ref{assp:blevelset}, the solution of the discretized problem converges to the optimal solution.
    More formally, we have
    \begin{align*}
        &\mu(B_i)\leq \mu(B_{i+1}) \leq \mu(B),\forall t\in\mathbb{N}_0\\
        &\lim_{i\rightarrow \infty} \mu(B_i) = \mu(B).
    \end{align*}
    If $\bold{c}^{\ast}$ is the unique optimal solution of the original problem, and $\bold{c}^{\ast}_{i}$ is the optimal solution of the discretization relaxation with grid $B_i$, then
    \begin{align*}
        \lim_{i \rightarrow \infty} ||\bold{c}^{\ast}-\bold{c}^{\ast}_{i}||_2 = 0.
    \end{align*}
\end{theorem}
It is straightforward to see that our chosen grid is arbitrary fine. Hence, we only need to prove that there exists a $\mathbf{c}_0$
such that the level set $\text{Level}(\mathbf{c}_0,D)$ is bounded.

Let $\mathbf{c} = (\mathbf{a};t)$ and note that in our setting, $f(\mathbf{c}) = t$. Rewrite $g(\mathbf{c},\lambda)$ in matrix form as
\begin{align*}
     g(\mathbf{c},\lambda) = \mathbf{a}^T\mathbf{M}(\lambda)\mathbf{a}+\mathbf{a}^T\Lambda\Lambda^T\mathbf{a}-t,
\end{align*}
where  $$\Lambda \triangleq e^{-\lambda}(\lambda^{0},\lambda^{1},...,\lambda^{L})^T.$$
Note that only $a_1,...a_L$ are allowed to vary since we fixed $a_0 = -1$. Obviously, $\Lambda\Lambda^T$ is positive semi-definite and the previously introduced $\mathbf{M}(\lambda)$ is positive definite for all
$\lambda>0$. Since the constraints on $g$ in~\eqref{sqlosseq5} are positive definite with respect to $a_1,...a_L,$ $g$ is coercive in $a_1,...a_L$. Furthermore, for any given $t$, the set of feasible coefficients $a_1,...a_L$ is bounded. Therefore, given a $t_0$, the level set $\text{Level}(\mathbf{c}_0,B_0)$ is bounded. This ensures that Assumption~\ref{assp:blevelset} holds for our optimization problem.

Next, we prove the uniqueness of the optimal solution $\mathbf{c}^\star$. Note that proving this result is
equivalent to proving the uniqueness of $\mathbf{a}^\star$. Hence, we once again refer to the original
minmax formulation of our problem,
\begin{equation}
    \inf_{\mathbf{a}:a_0=-1}\sup_{\lambda\in [\frac{n}{k},6.5L]}\mathbf{a}^T(\mathbf{M}(\lambda)+\mathbf{\Lambda}\mathbf{\Lambda}^T)\mathbf{a}\triangleq \inf_{\mathbf{a}:a_0=-1}\sup_{\lambda\in [\frac{n}{k},6.5L]} h_{\lambda}(\mathbf{a})
\end{equation}
Clearly, $\forall \lambda\in [\frac{n}{k},6.5L],$ the function $h_{\lambda}(\mathbf{a})$ is strictly convex since
$(\mathbf{M}(\lambda)+\mathbf{\Lambda}\mathbf{\Lambda}^T) \succ 0,\;\forall \lambda\in [\frac{n}{k},6.5L].$
Taking the supremum over $\lambda$ preserves strict convexity since $\forall \theta\in(0,1),$ one has
\begin{align*}
    &\sup_{\lambda\in [\frac{n}{k},6.5L]} h_{\lambda}(\theta\mathbf{x}+(1-\theta)\mathbf{y})\\
    &< \sup_{\lambda\in [\frac{n}{k},6.5L]} \theta h_{\lambda}(\mathbf{x})+(1-\theta)h_{\lambda}(\mathbf{y}) \\
    &\leq \sup_{\lambda\in [\frac{n}{k},6.5L]} \theta h_{\lambda}(\mathbf{x})+\sup_{\lambda'\in [\frac{n}{k},6.5L]}(1-\theta)h_{\lambda'}(\mathbf{y}).
\end{align*}
Hence $\sup_{\lambda\in [\frac{n}{k},6.5L]} h_{\lambda}(\mathbf{a})$ is strictly convex, which consequently implies the uniqueness of $\mathbf{a}^\star$ and hence $\mathbf{c}^\star$. This proves the convergence result of Theorem~\ref{thm:discretization}.

\subsection{Proof for the convergence rate}
In what follows, and for reasons of simplicity, we omit the constraint $a_0 = -1$ in the SIP formulation. The described proof only requires small modifications to accommodate the constraint $a_0 = -1$.

Recall that we used $B_d$ to denote the grid with grid spacing $d$. In order to use the results in~\cite{still2001discretization}, we require the convergence assumption below.

\begin{assumption}\label{assump:rate1}
    Let $\bar{\mathbf{c}}$ be a local minimizer of an SIP. There exists a local solution $\mathbf{c}_d$ of the discretized SIP with grid $B_d$ such that $$||\mathbf{c}_d-\bar{\mathbf{c}}|| \rightarrow 0.$$
\end{assumption}
This assumption is satisfied for the SIP of interest as shown in the first part of the proof.

\begin{assumption}\label{assump:rate2}
The following hold true:
    \begin{itemize}
        \item There is a neighborhood $\bar{U}$ of $\bar{\mathbf{c}}$ such that the function $\frac{\partial^2}{\partial \lambda^2}g(\mathbf{c},\lambda)$ is continuous on $\bar{U}\times B$.
        \item The set $B$ is compact, non-empty and explicitly given as the solution set of a set of inequalities, $B = \{\lambda\in \mathbb{R}| v_i(\lambda)\leq 0,i\in I\},$ where $I$ is a finite index set and $v_i\in C^2(B)$.
        \item For any $\bar{\lambda}\in B$, the vectors $\frac{\partial}{\partial \lambda}v_i(\bar{\lambda}),i\in \{i\in I| v_i(\bar{\lambda}) = 0\}$ are linearly independent.
    \end{itemize}
\end{assumption}
Recall that our objective is of the form
\begin{align*}
     g(\mathbf{c},\lambda) = \mathbf{a}^T\mathbf{M}(\lambda)\mathbf{a}+\mathbf{a}^T\Lambda\Lambda^T\mathbf{a}-t,
\end{align*}
where
\begin{align*}
    &\Lambda \triangleq e^{-\lambda}(\lambda^{0},\lambda^{1},...,\lambda^{L})^T,\;\mathbf{c} = (\mathbf{a};t),\\
    &\mathbf{M}(\lambda) \triangleq \frac{e^{-\lambda}}{k} Diag(\lambda^{0}0!,\lambda^{1}1!,...,\lambda^{L}L!).
\end{align*}
It is straightforward to see that the first condition in Assumption~\ref{assump:rate2} holds. For the second condition, recall that $B = [\frac{n}{k}, 6.5L]$. Hence, the second condition can be satisfied by choosing $I = \{1\}$, $v_1(\lambda) = (\lambda-\frac{n}{k})(\lambda-6.5L)$. Since we only have one variable $v_1$, it is also easy to see that the third condition is met.
\begin{assumption}\label{assump:rate3}
    The set $B$ satisfies Assumption~\ref{assump:rate2} and all the sets $B_d$ contain the boundary points $\frac{n}{k}, 6.5L$.
\end{assumption}
This assumption also clearly holds for the grid of choice. Note that it is crucial to include the boundary points for the proof in~\cite{still2001discretization} to be applicable.
\begin{assumption}\label{assump:rate4}
    $\nabla_{\mathbf{c}} g(\mathbf{c},\lambda)$ is continuous on $\bar{U}\times B$, where $\bar{U}$ is a neighborhood of $\bar{\mathbf{c}}$. Moreover, there exists a vector $\xi$ such that $$\nabla_{\mathbf{c}}g(\bar{\mathbf{c}},\lambda)^T\xi\leq -1,\;\forall \lambda\in B.$$
\end{assumption}
Note that $\nabla_{\mathbf{c}}g(\mathbf{c},\lambda) = [\nabla_{\mathbf{a}}g(\mathbf{c},\lambda);\nabla_{t}g(\mathbf{c},\lambda)]$ and
$$\nabla_{\mathbf{a}}g(\mathbf{c},\lambda) = 2(\mathbf{M}(\lambda)+\Lambda\Lambda^T)\mathbf{a}.$$
Also note that $\forall \lambda\in B$, $\mathbf{M}(\lambda)+\Lambda\Lambda^T$ is positive definite. Hence choosing $\xi$ to be colinear with and of the same direction as $[-\mathbf{a}^T \; 1]^T$, as well as of sufficiently large norm will allow us to satisfy the inequality
$$\nabla_{\mathbf{c}}g(\bar{\mathbf{c}},\lambda)^T\xi\leq -1,\;\forall \lambda\in B.$$
Hence, Assumption~\ref{assump:rate4} holds as well. The next results follow from the above assumptions and observations, and the results in~\cite{still2001discretization}.
\begin{lemma}[Corollary 1 in~\cite{still2001discretization}]\label{sup:discreterate1}
    Let $t_d$ be the optimal objective value of the discretized SIP used for support estimation with the grid $B_d,$ and let $t^\star$ be the optimal objective value for the original SIP. Since Assumptions~\ref{assump:rate1},\ref{assump:rate2},\ref{assump:rate3},\ref{assump:rate4} hold, then for some $c_3>0$ and $d$ sufficiently small, we have
    $$0\leq t^\star-t_d \leq c_3d^2.$$
\end{lemma}
Consequently, $t_d\rightarrow t^\star$ with a convergence rate of $O(d^2)$.
\begin{lemma}[Theorem 2 in~\cite{still2001discretization}]
\label{sup:discreterate2}
    Assume that all assumptions in Lemma~\ref{sup:discreterate1} hold. If there exists a constant $c_4>0$ such that $$t-\bar{t} \geq c_4||\mathbf{c}-\bar{\mathbf{c}}||,\;\forall \mathbf{c}\in M(B)\cap \bar{U},$$ then for sufficiently small $d$ and $\sigma>0$ we have
    $$||\mathbf{c}_d-\bar{\mathbf{c}}||\leq \sigma d^2.$$
\end{lemma}
This result implies that if $\bar{\mathbf{c}}$ is also a strict minimum of order one, then the solution of the discretized SIP converges to that of the the original SIP with rate $O(d^2)$. Combining these results completes the proof.

\section{Theoretical results supporting Remark~\ref{remark:pi_interval}}
The result described in the main text follows from Theorem 6.2 in \cite{lubinsky2007survey}, originally proved in~\cite{mhaskar1984extremal,mhaskar1985does} and~\cite{rakhmanov1984asymptotic}.
\begin{theorem}[Theorem 6.2 from~\cite{lubinsky2007survey}]\label{thm:MRS}
    Let $W(x) = \exp(-Q(x))$ be a weight function, where $Q:\mathbb{R} \mapsto [0,\infty) $ is even, convex, diverging for $x \to \infty$, and such that
    \begin{equation*}
        0 = Q(0) < Q(x), \forall x \neq 0.
    \end{equation*}
    Then, for any polynomial $P(x)$ of degree $\leq L$, not identical to zero, one has
    \begin{align*}
        &\sup_{x\in \mathbb{R}}|P(x)W(x)| = \sup_{x\in [-M_L,M_L]}|P(x)W(x)|,\\
        &\sup_{x\in \mathbb{R}\setminus [-M_L,M_L]}|P(x)W(x)| < \sup_{x\in [-M_L,M_L]}|P(x)W(x)|.
    \end{align*}
    Here, $M_L$ stands for the \textit{Mhaskar-Rakhmanov-Saff} (MSF) number, which is the smallest positive root of the integral equation
    \begin{equation}\label{app:MRSnumber}
        L = \frac{2}{\pi}\int_{0}^{1}\frac{M_LtQ'(M_Lt)}{\sqrt{1-t^2}}dt.
    \end{equation}
\end{theorem}
In our setting, the weight equals $\exp(-x)$. Solving~\eqref{app:MRSnumber} gives us an MSF number equal to $M_L = \frac{\pi}{2}L$. Thus, we can restrict our optimization interval to $[\frac{n}{k},\frac{\pi}{2}L +\frac{n}{k}]$. If there is no regularization term, the optimal interval reduces to $[\frac{n}{k},\frac{\pi}{2}L +\frac{n}{k}]$.

\section{Construction of the RWC-S estimator}
We introduce the optimization problem needed for minimizing the risk $\E \left( \frac{S-\hat{S}}{S}\right)^2$. Poissonization arguments once again establish that
\begin{align*}
    &\mathbb{E}\left(\frac{S-\hat{S}}{S}\right)^2 = \frac{1}{S^2}\bigg\{\sum_{i:\lambda_i>0}\bigg(\sum_{l=0}^{L}e^{-\lambda_i}a_l^2\lambda_i^ll!\bigg)\\
    &+\sum_{i\neq j:\lambda_i\lambda_j>0}\bigg(e^{-\lambda_i}P_L(\lambda_i,\mathbf{a})\bigg)\bigg(e^{-\lambda_j}P_L(\lambda_j,\mathbf{a})\bigg)\bigg\}.
\end{align*}
Taking the supremum over $D_k$, one can further upper bound the risk as
\begin{align}
    &\leq \sup_{\mathbf{\lambda}:\lambda_i\in [\frac{n}{k}, n]} \frac{1}{S^2}\bigg\{\sum_{i:\lambda_i>0}\bigg(\sum_{l=0}^{L}e^{-\lambda_i}a_l^2\lambda_i^ll!\bigg)\\
    &+\sum_{i\neq j:\lambda_i\lambda_j>0}\bigg(e^{-\lambda_i}P_L(\lambda_i,\bold{a})\bigg)\bigg(e^{-\lambda_j}P_L(\lambda_j,\mathbf{a})\bigg)\bigg\}\nonumber\\
    & \leq \sup_{\lambda\in [\frac{n}{k}, n]}\bigg\{\frac{1}{S}\bigg(\sum_{l=0}^{L}e^{-\lambda}a_l^2\lambda^ll!\bigg)+\bigg(e^{-\lambda}P_L(\lambda,\mathbf{a})\bigg)^2\bigg\}\nonumber\\
    &\leq \sup_{\lambda\in [\frac{n}{k}, n]}\bigg\{\frac{1}{\hat{S}_{c}}\bigg(\sum_{l=0}^{L}e^{-\lambda}a_l^2\lambda^ll!\bigg)+\bigg(e^{-\lambda}P_L(\lambda,\mathbf{a})\bigg)^2\bigg\} \label{ratiolosseq3},
\end{align}
where the last inequality is due to the fact that $\hat{S}_{c}\leq S$. Note that the only difference between~\eqref{ratiolosseq3} and~\eqref{sqlosseq1}
is in terms of changing $1/k$ to $1/\hat{S}_c$ in the first term. In view of Theorem~\ref{lma:interval_sq}, \eqref{ratiolosseq3} is optimized by the solution of the following problem:

\begin{equation}\label{ratiolosseq2}
    \begin{split}
        &\min_{t,a_1,...,a_L} t \;\;\;subject\;to\\
        & \bigg\{\frac{1}{\hat{S}_{c}}\bigg(\sum_{l=0}^{L}e^{-\lambda}a_l^2\lambda^ll!\bigg)+\bigg(e^{-\lambda}P_L(\lambda,\mathbf{a})\bigg)^2\bigg\}\leq t\\
        & \forall \lambda\in \text{Grid}([\frac{n}{k}, 6.5L],s), \text{with }a_0=-1.
    \end{split}
\end{equation}
\ifCLASSOPTIONcaptionsoff
  \newpage
\fi
\end{document}